\documentclass{article}

\usepackage[preprint]{neurips_2026}

\usepackage[utf8]{inputenc}
\usepackage[T1]{fontenc}
\usepackage{hyperref}
\usepackage{url}
\usepackage{booktabs}
\usepackage{amsfonts}
\usepackage{nicefrac}
\usepackage{microtype}
\usepackage{xcolor}
\usepackage{amsmath}
\usepackage{amssymb}
\usepackage{amsthm}
\usepackage{graphicx}
\usepackage{multirow}
\usepackage{array}
\usepackage{algorithm}
\usepackage{algpseudocode}
\usepackage{enumitem}
\usepackage{subcaption}

\newtheorem{proposition}{Proposition}
\newtheorem{lemma}{Lemma}

\newcommand{\method}{Aperture}
\newcommand{\E}{\mathbb{E}}

\newcommand{\cfail}{c^{\mathrm{fail}}}
\newcommand{\thprobe}{\theta_{\mathrm{probe}}}
\newcommand{\poolMath}{17{,}398}
\newcommand{\poolLogic}{6{,}400}
\newcommand{\grp}{8}
\newcommand{\nprobe}{512}
\newcommand{\epProbe}{6}
\newcommand{\epMain}{8}
\newcommand{\mixLam}{0.75}
\newcommand{\mixAlpha}{0.5}
\newcommand{\nEval}{4}
\newcommand{\probeOverM}{2.2\%}
\newcommand{\evalOver}{12.5\%}
\newcommand{\totOverM}{14.7\%}
\newcommand{\totOverL}{18.5\%}
\newcommand{\mAccG}{0.377}
\newcommand{\mAccA}{0.395}
\newcommand{\mAccD}{0.385}
\newcommand{\mAccDA}{0.401}

\newcommand{\MAccG}{0.481}
\newcommand{\MAccA}{0.497}

\newcommand{\MAccDA}{0.502}
\newcommand{\lAccG}{0.171}
\newcommand{\lAccA}{0.185}

\newcommand{\LAccG}{0.211}
\newcommand{\LAccA}{0.223}

\newcommand{\mLiveG}{45.8\%}
\newcommand{\mLiveA}{60.7\%}
\newcommand{\mLiveDA}{62.8\%}

\newcommand{\mSTwoBA}{68.5}

\newcommand{\costD}{1.48}
\newcommand{\costDA}{1.23}
\newcommand{\costA}{1.02}
\newcommand{\binLoAccLo}{0.328}
\newcommand{\binLoAccHi}{0.401}
\newcommand{\binLoLiveLo}{28.6\%}
\newcommand{\binLoLiveHi}{47.1\%}
\newcommand{\abHeadAcc}{0.388}
\newcommand{\abHeadLive}{51.9\%}
\newcommand{\abNuAcc}{0.384}
\newcommand{\abNuLive}{57.2\%}
\newcommand{\failAccG}{0.742}
\newcommand{\failAccA}{0.739}
\newcommand{\xferCross}{0.490}
\newcommand{\xferDom}{0.176}
\newcommand{\spearEnd}{0.610}
\newcommand{\spearEarly}{0.840}
\newcommand{\witCscore}{0.001}
\newcommand{\witChg}{0.166}
\newcommand{\witAscore}{0.412}

\title{Where Groups Split: A Probe Prior for Nonzero Within-Group Variance in RL Post-Training}

\author{%
  Jonah Ellison \\
  University of Michigan \\
  \And
  Mateo Ruiz \\
  University of Michigan \\
  \And
  Owen Pell \\
  University of Michigan \\
  \AND
  Tomas Duncun \\
  University of Michigan \\
  \And
  Hanae Fujita \\
  University of Edinburgh \\
  \And
  Ingrid Solberg \\
  University of Edinburgh \\
  \And
  Caleb Nguyen \\
  University of Edinburgh
}

\begin{document}

\maketitle

\begin{abstract}
Group relative policy optimization learns from a prompt only when the rewards in the sampled group disagree.
With a binary verifier, a group that is entirely correct or entirely incorrect has within-group variance zero and a zero advantage.
Uniform sampling continues to spend rollouts on prompts that are unanimous now and stay unanimous as training goes on.
Pass rate is a weak forecast of that future.
Two prompts at the same pass rate can move apart: one into the range where a finite group splits, and one toward a ceiling where the group becomes unanimous again.
\method{} estimates the later pass rate from a short probe.
A small subset of the pool is trained with the same recipe, every prompt is scored on the base model and on the probe model, and the two pass rates are plugged into the probability that a group of the training size is all wrong or split.
The resulting weight is a static prior.
It can be used alone or multiplied by an online scheduler, and the GRPO update is left as it is.
On mathematical and logical reasoning with 1.7B and 4B models, the prior improves greedy accuracy over uniform GRPO and raises the fraction of groups with a nonzero advantage.
The same prior still helps when it is combined with dynamic sampling or with selective rollouts.
\end{abstract}

\section{Introduction}
\label{sec:intro}

Reinforcement learning with verifiable rewards is now a standard way to elicit reasoning from language models \citep{deepseekai2025r1,shao2024deepseekmath,yu2025dapo,yang2025qwen3}.
The usual update, group relative policy optimization (GRPO), draws several responses for one prompt and centers their rewards.
That centered residual is the advantage.
If every reward in the group is the same, the residual is the zero vector and the prompt does not contribute to the policy gradient.
With a binary verifier, this is a unanimous group.
Unanimous groups are common on hard pools, where most prompts are still failed by every sample, and they show up again late in training, once easy prompts are solved by every sample.

Recent methods already try to stop spending updates on these groups.
Dynamic sampling redraws a prompt until the group contains both outcomes \citep{yu2025dapo}.
Selective rollout skips a prompt before generation when the prompt looks uninformative \citep{zheng2025greso}.
Other work keeps the unanimous prompt and reshapes its advantage, or replaces a difficulty guess that uses no rollout with a filter that runs online \citep{le2025noprompt,sha2026thinkprior,jiang2025vcrl}.
These rules look at the group just drawn, or at a snapshot of how hard the prompt looks under the current policy.
Which of today's unanimous prompts will still be unanimous after the policy has moved is left open.

A prompt's pass rate is the chance that the current policy solves it.
For a group of size $G$, the unanimity probability is the fixed function $p^G+(1-p)^G$, and it is large near $0$ and near $1$.
Two prompts with the same pass rate therefore have the same chance of a dead group today.
After further training they can separate.
One can move off the floor into the range where groups split.
Another can climb from an already mixed rate toward the ceiling, so the groups seal and the advantage disappears while accuracy is still rising.
A score that rewards accuracy gain prefers the second prompt.
A score that rewards a future split prefers the first.

\method{} is a static prior for prompts of the first kind.
Before the main run, we train the same algorithm for a few epochs on a few hundred prompts drawn from the pool.
We then evaluate every prompt twice, once under the base model and once under the probe model, and we shrink the two empirical pass rates with a Jeffreys correction.
The score of a prompt is the product of two closed forms at the training group size: the probability that a group drawn at the probe pass rate is non-unanimous, and the drop in the probability that the group is entirely wrong.
Prompts whose pass rate falls, and prompts that arrive near the ceiling, receive a zero or near-zero score.
The sampling distribution mixes this score with the uniform distribution, so prompts that are already interior keep a floor of mass.
The same weight can multiply an online scheduler without changing the scheduler's update.

The prior is cheap relative to the main run, and it is tied to the recipe that produced it.
On the math pool the probe plus the two endpoint passes cost \totOverM{} of the main-run rollouts.
A prior fit to one model, one group size, and one reward has to be checked again if any of those change.
Section~\ref{sec:theory} records the identity that makes the score well-defined, and a small witness on which a headroom-gain score and \method{} disagree.
Section~\ref{sec:exp} tests the claim that this disagreement matters for accuracy and for the fraction of groups that actually split.

\paragraph{Contributions.}
\begin{itemize}[leftmargin=1.1em,itemsep=0.25em,topsep=0.3em]
  \item We treat future non-unanimity as the quantity that decides whether a prompt will yield a GRPO advantage, and we give the exact unanimity probability for binary groups together with the group-size dependence of that probability.
  \item We introduce \method{}, a probe estimator whose score multiplies the probe-future probability of a split group by the drop in the all-wrong probability, with a Jeffreys shrink on the endpoint pass rates and a uniform mixture on the weights.
  \item On math and logic, at 1.7B and 4B, the prior improves uniform GRPO and composes with dynamic sampling and selective rollout.
  Matched-pass-rate cuts, score ablations, and a constructed witness separate the effect from accuracy gain.
\end{itemize}

\section{Related work}
\label{sec:related}

\paragraph{Dead groups inside GRPO.}
GRPO centers binary or scalar rewards inside a group and optionally divides by the group standard deviation \citep{shao2024deepseekmath}.
The division is undefined or unstable when the group does not vary, and removing it does not create an advantage where the rewards are identical \citep{liu2025drgrpo,hu2025reinforcepp}.
Analyses of the GRPO loss make the same point from the objective side: the update is a contrast among the samples in the group, so a group with one repeated reward has nothing to contrast \citep{mroueh2025grpo,zheng2025gspo}.
Dynamic sampling, as used in large GRPO systems, discards the unanimous group and draws another \citep{yu2025dapo,yue2025vapo}.
The draw is honest about the current group and expensive when the prompt is sealed, because the redraws repeat until a split appears or a cap is hit.
Entropy-guided shaping and related advantage corrections try to extract a signal from the unanimous prompt itself \citep{le2025noprompt,wang2025entropy}.
\method{} does not modify the advantage.
It changes which prompts are offered to an unchanged GRPO update.

\paragraph{Online selection.}
A number of schedulers reweight prompts while training.
Moving-edge filters keep prompts whose recent rollouts sit near the decision boundary \citep{wu2026edge}.
\citet{bae2026filter} skip prompts that an online predictor marks as too easy or too hard.
\citet{mao2026dynamics} sample from a forecast of how a prompt's training value will move, rather than from its value at the current step, and \citet{qu2026small} fit a small predictor that can be reused across runs to decide which prompts deserve rollouts.
Selective rollout spends generation on prompts that are predicted to repay it \citep{zheng2025greso}.
Curricula from easy to hard, bandit curators, and gradient-based prompt scores make related online choices \citep{parashar2025curriculum,chen2025sec,gu2026actor,zeng2025cures}.
Once a run has history, medium-difficulty prompts can be replayed and pre-rollout screens can drop prompts that have stopped mixing successes and failures \citep{liang2026leeps,baroian2026replay,liu2025explore}.
\citet{mao2026rlvr} carry a similar filter into the update: groups that come back all correct or all wrong are treated as ineffective samples and down-weighted with an off-policy priority.
For multi-step agents the budget also has to be split across steps.
TRACE assigns a rollout budget over the steps of a search agent \citep{zou2026trace}, and zero-variance search queries have been recycled rather than dropped \citep{coelho2026recycle}.
All of these methods observe the policy that is training.
Early on, that observation is a weak forecast of which sealed prompts will open, and a prompt that is mixed today can seal later.
\method{} is computed once, from a probe of the same recipe, and is then held fixed or used as a multiplier on the online weight.

\paragraph{Static scores and what they miss.}
Offline filters rank a pool before training, often by the base model's pass rate, by a length or difficulty feature, or by a small supervised signal \citep{li2025limr,ye2025limo,shi2025adarf,sha2026thinkprior}.
Pass rate is the right sufficient statistic for the \emph{current} unanimity probability, and a poor one for the change in that probability.
Training dynamics papers make a related observation: many prompts stay solved or unsolved for the whole run, so a single snapshot hides the prompts that move \citep{foster2025frontier,zeng2025simplrl}.
Variance-based curricula use the spread of rewards as a difficulty coordinate and walk along it \citep{jiang2025vcrl}.
That spread, at a fixed time, is again a function of the current pass rate.
\method{} uses two times, and it uses them only through the unanimity formulas at the group size the main run will actually draw.
The estimator is therefore allowed to prefer a small accuracy gain that enters the interior over a large accuracy gain that exits it.

Adjacent evidence says the gradient of a reasoning policy is uneven in other ways as well.
A minority of high-entropy tokens can dominate the update \citep{wang2025entropy}, policy entropy can collapse even while accuracy is still moving \citep{cui2025entropy}, and verifiable rewards can be partly spurious \citep{shao2025spurious,yue2025incentivize}.
Those are reasons to be careful about what a binary group means.
They do not remove the mechanical fact that a unanimous group has a zero centered advantage.
We hold the reward, the advantage, and the optimizer fixed, and we only change the prompt distribution.
Broader recipes that change the loss, the horizon, or the mixture across domains sit outside that comparison \citep{liu2025prorl,lin2025cppo,zeng2025eachprompt,li2025multidomain}.

\section{When a group teaches nothing}
\label{sec:theory}

\subsection{Unanimity is a function of the pass rate}
\label{sec:identity}

Let a prompt have unknown pass rate $p\in[0,1]$ under the current policy, and let a group of $G\ge 2$ rewards be drawn i.i.d.\ from $\mathrm{Bernoulli}(p)$.
Write $m$ for the group mean.
The group is unanimous when $r_1=\cdots=r_G$, which is exactly when the within-group variance is zero.

\begin{lemma}[Dead groups]
\label{lem:dead}
The group is unanimous with probability
\begin{equation}
\label{eq:collapse}
c_G(p)=p^G+(1-p)^G,
\end{equation}
and on that event every centered advantage $a_i\propto(r_i-m)$ is zero.
The probability that the group is entirely wrong is $\cfail_G(p)=(1-p)^G$.
The probability that it splits, and therefore that the within-group variance is positive, is
\begin{equation}
\label{eq:live}
\nu_G(p)=1-c_G(p).
\end{equation}
\end{lemma}

The expected residual sum of squares is $(G-1)p(1-p)$ (Lemma~\ref{lem:energy} in the appendix), so expected energy and the probability of a dead group are different summaries.
Both are minimized or maximized at $p=1/2$, and for a \emph{single} snapshot they induce the same ranking of prompts.
The ranking that matters for allocation is the ranking of \emph{pairs} $(p_0,p_1)$, and that ranking depends on $G$.

\begin{proposition}[Shape and group size]
\label{prop:shape}
For $G\ge 2$, $c_G$ is strictly decreasing on $[0,1/2]$ and strictly increasing on $[1/2,1]$.
For every $p\in(0,1)$ and every $G\ge 1$, $c_{G+1}(p)<c_G(p)$.
\end{proposition}

Larger groups hide the problem: a pass rate of $0.2$ is usually unanimous at $G=4$ and rarely unanimous at $G=16$ (Table~\ref{tab:grid} in the appendix).
A prior that plugs in the wrong group size is a prior for a different experiment.
Figure~\ref{fig:mech}a plots the identity.
The curve at the training size $G=\grp$ is the thick one.

\subsection{The same pass rate, two futures}
\label{sec:futures}

Consider three prompts with population pass rates, before any shrink, and the natural headroom-gain score $s_{\mathrm{hg}}(p_0,p_1)=(1-p_0)[p_1-p_0]_+$.
Prompt A moves from $0.05$ to $0.20$.
Prompt B moves from $0.25$ to $0.55$.
Prompt C moves from $0.55$ to $0.92$.
Headroom-gain ranks C above A at every group size, because C gains $0.37$ accuracy from a pass rate that still has headroom.
At $G=8$, which is the group size we train with, \method{}'s population score (defined in Section~\ref{sec:method}) assigns A the value \witAscore{} and C the value \witCscore{}, against headroom-gain values $0.143$ and \witChg{}.
C is the prompt that climbs into the ceiling.
Its groups stop splitting.
A is the prompt that leaves the floor.
The full witness, including a reversal between A and B as $G$ goes from $4$ to $16$, is Proposition~\ref{prop:reverse} and Table~\ref{tab:witness}.

The same pattern appears in a probe.
Figure~\ref{fig:mech}b plots endpoint pass rates for a diagnostic draw from the math pool, colored by the \method{} score.
Points that start near the origin and move up are dark.
Points that are already high and move higher are pale.
Figure~\ref{fig:mech}c follows the all-wrong probability of the top and bottom quartiles of the score through the main run.
The two quartiles start near each other.
The high-score quartile's all-wrong probability falls, and the low-score quartile stays sealed.
The base pass rate is only the horizontal axis in panel (b).
The score also depends on where the prompt goes.

\begin{figure}[t]
  \centering
  \includegraphics[width=\textwidth]{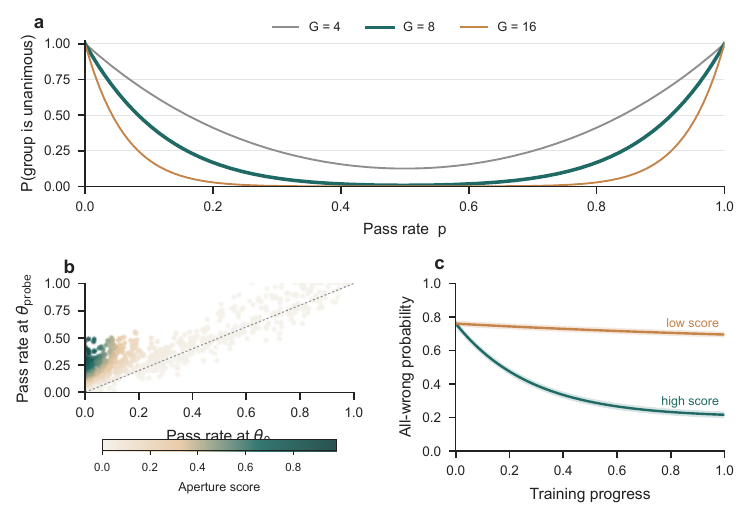}
  \caption{Unanimity, and which prompts leave it.
  (a)~Exact probability that $G$ i.i.d.\ Bernoulli rewards are unanimous.
  The training group size is $G=\grp$.
  (b)~Endpoint pass rates on a diagnostic draw, colored by the \method{} score.
  Darker points are prompts the score will upweight.
  (c)~All-wrong probability during the main math run for the top and bottom quartiles of the score, on Qwen3-1.7B.
  Bands are standard deviations across three seeds.}
  \label{fig:mech}
\end{figure}

\section{\method{}}
\label{sec:method}

\subsection{A short probe and two endpoints}
\label{sec:probe}

Let $D$ be the training pool and let $P\subset D$ be a uniform subset with $|P|=\nprobe$.
Starting from the base weights $\theta_0$, we run GRPO on $P$ for $\epProbe$ epochs with the same group size, reward, and optimizer the main run will use.
Call the result $\thprobe$.
The probe is not a curriculum over $D$.
It is a cheap exposure of the base model to the optimization signal that $D$ produces, so that $\thprobe$ stands in for a later policy.
We then draw $\nEval$ rollouts on every $x\in D$ under $\theta_0$ and under $\thprobe$.
If $k$ of those rollouts are correct, the endpoint pass rate is the Jeffreys estimator
\begin{equation}
\label{eq:jeffreys}
\tilde\rho=\frac{k+\tfrac12}{\nEval+1}.
\end{equation}
Write $\tilde\rho_0(x)$ and $\tilde\rho_1(x)$ for the two endpoints.
Four failures out of four are not treated as a pass rate of zero.
With $\nEval$ evaluation rollouts the empty count maps to $1/(2(\nEval+1))=0.1$, and the implied all-wrong probability at $G=\grp$ falls from $1$ to $0.9^{8}=0.430$ (Appendix~\ref{app:shrink}).
The ablation that removes the shrink loses accuracy, so the correction is doing work beyond smoothing a plot.

\subsection{The score}
\label{sec:score}

\begin{equation}
\label{eq:score}
s(x)=\nu_G\bigl(\tilde\rho_1(x)\bigr)\,
\Bigl[\cfail_G\bigl(\tilde\rho_0(x)\bigr)-\cfail_G\bigl(\tilde\rho_1(x)\bigr)\Bigr]_+.
\end{equation}
The first factor is the probability that a group of the training size splits at the probe pass rate.
The second factor is the drop in the probability of an all-wrong group.
Their product is large when the probe moves the prompt off the floor and the destination is still a split, and it is small when the destination is a ceiling.
Because $\cfail_G$ is strictly decreasing on $[0,1]$, the positive part is zero whenever $\tilde\rho_1(x)\le\tilde\rho_0(x)$.
A prompt whose pass rate falls receives score zero, including one that falls from the ceiling back into the interior.
We leave that prompt at the uniform floor, on the view that a falling probe pass rate is as likely to be forgetting as it is to be a useful return to the mixed range.

The estimator is the product in \eqref{eq:score}.
Training does not fit a separate scorer or value model, and $s$ stays fixed during the main run.
The group size inside $\nu_G$ and $\cfail_G$ is the group size of the main run.
$\nEval$ is only the number of rollouts used to estimate each pass rate.

\subsection{Weights, and the online product}
\label{sec:weights}

Scores are skewed, and a pure proportional sample collapses onto a handful of prompts.
We mix the powered score with the uniform distribution,
\begin{equation}
\label{eq:mix}
w(x)=\frac{1-\lambda}{|D|}+\lambda\,\frac{s(x)^{\alpha}}{\sum_{x'\in D}s(x')^{\alpha}},
\end{equation}
with $\lambda=\mixLam$ and $\alpha=\mixAlpha$.
If every score is zero, $w$ is uniform.
Every prompt keeps probability at least $(1-\lambda)/|D|$, which is the mass we deliberately leave on prompts that are already interior and therefore have little opening left to score.
An online scheduler with weight $w_{\mathrm{on}}(x,t)$ is composed by multiplication,
\begin{equation}
\label{eq:product}
w_{\mathrm{final}}(x,t)=w(x)\,w_{\mathrm{on}}(x,t),
\end{equation}
renormalized over the pool.
The weight is therefore available from the first step, and the online term can take over as evidence from the main run accumulates.
Algorithm~\ref{alg:aperture} puts the steps in order.
The main-run GRPO loss, clipping, and KL penalty are exactly the baseline's.

\begin{algorithm}[t]
\caption{\method{} prior, then GRPO}
\label{alg:aperture}
\begin{algorithmic}[1]
\Require Pool $D$, probe size $\nprobe$, epochs $\epProbe$, group size $G$, eval rollouts $\nEval$, mixture $\lambda$, power $\alpha$
\State Draw $P\subset D$ uniformly, $|P|=\nprobe$
\State Train $\thprobe$ by GRPO on $P$ for $\epProbe$ epochs, starting from $\theta_0$
\For{$x\in D$}
  \State Estimate $\tilde\rho_0(x),\tilde\rho_1(x)$ by \eqref{eq:jeffreys} with $\nEval$ rollouts
  \State $s(x)\leftarrow$ \eqref{eq:score}
\EndFor
\State $w\leftarrow$ \eqref{eq:mix}
\State Train on $D$ for $\epMain$ epochs with prompt distribution $w$, or with $w\cdot w_{\mathrm{on}}$
\end{algorithmic}
\end{algorithm}

\begin{figure}[t]
  \centering
  \includegraphics[width=\textwidth]{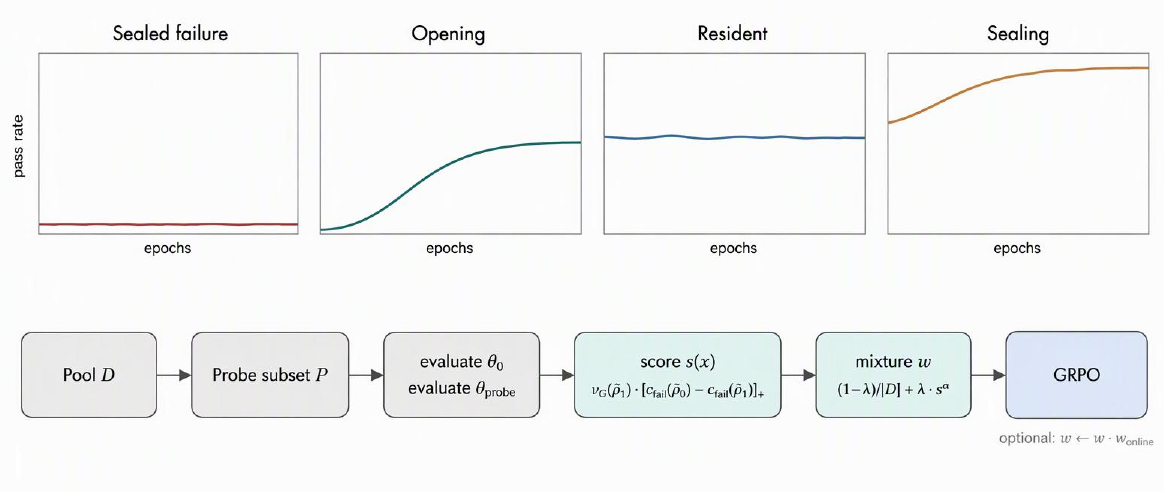}
  \caption{The \method{} prior.
  The top row sketches four pass-rate trajectories the probe is meant to tell apart: sealed failure, opening, a resident interior prompt, and sealing into the ceiling.
  The bottom row is the computation: a uniform probe, two endpoint evaluations, the score in \eqref{eq:score}, the mixture in \eqref{eq:mix}, and GRPO.
  An online scheduler, when one is used, multiplies $w$ and does not replace it.}
  \label{fig:overview}
\end{figure}

\section{Experiments}
\label{sec:exp}

\subsection{Setup}
\label{sec:setup}

We train two verifiable pools.
Math uses the DAPO-Math release, \poolMath{} integer-answer problems \citep{yu2025dapo}.
Logic uses \poolLogic{} prompts: Knights-and-Knaves puzzles from Logic-RL \citep{xie2025logicrl} together with ordering and graph puzzles from the logic slice of GURU \citep{cheng2025guru}.
ARC-AGI-2 is held out of training and used only as a transfer evaluation \citep{chollet2025arc2}.
The reward is exact match, $1$ or $0$, with no partial credit.
That is the Bernoulli setting of Lemma~\ref{lem:dead}.

Math evaluation is greedy pass@1 on MATH-500 \citep{hendrycks2021math}, Minerva \citep{lewkowycz2022minerva}, AIME 2024, AIME 2025, OlympiadBench \citep{he2024olympiad}, and AMC 2023.
Logic evaluation is greedy pass@1 on held-out Knights-and-Knaves at two size bands, held-out ordering, held-out graph puzzles, and ARC-AGI-2.
We report the unweighted macro-average.
Best accuracy is the peak of a five-point sliding window on checkpoints taken every tenth of training.
AUC is the mean of that checkpoint curve.
Live is the fraction of generated training groups whose rewards are not unanimous, averaged over the run.
S2B is the percent of training progress at which a prior-augmented run first matches the best accuracy of its paired baseline.
Lower is faster.
Uniform has no paired baseline, so its S2B cell is empty.

The main run is $\epMain$ epochs, prompt batch $64$, $G=\grp$ rollouts, learning rate $10^{-6}$, KL coefficient $10^{-3}$, and clip range $0.2$.
Generation length is $3072$ tokens on math and $2048$ on logic.
The probe uses the same hyperparameters for $\epProbe$ epochs on $\nprobe$ prompts.
Models are Qwen3-1.7B and Qwen3-4B \citep{yang2025qwen3}.
A Llama-3.2-3B run is in Appendix~\ref{app:llama} \citep{grattafiori2024llama}.
Training is implemented in Trinity-RFT on the HybridFlow GRPO stack \citep{pan2025trinity,sheng2024hybridflow}, on $8\times$A100-80GB.
Every reported mean is over three seeds.
Appendix~\ref{app:protocol} records the remaining choices: decoding, answer parsing, the held-out splits, and the overhead arithmetic behind \totOverM{}.

\paragraph{Baselines.}
Uniform GRPO draws prompts uniformly and keeps every group, unanimous or not.
Dynamic sampling is the DAPO rule with the rest of the update held fixed: a unanimous group is discarded and the prompt is redrawn, up to a cap, so the optimizer sees only split groups \citep{yu2025dapo}.
Selective rollout skips prompts a pre-rollout filter marks as uninformative, following the public recipe of \citet{zheng2025greso}.
\method{} is crossed with each of them by \eqref{eq:product}.
We do not retune the baselines' internal thresholds when the prior is attached.
Appendix~\ref{app:compose} adds the same product on top of a moving-edge filter.

\subsection{Main results}
\label{sec:mainresults}

\begin{table}[t]
  \caption{Math. Greedy macro-accuracy, area under the checkpoint curve, fraction of non-unanimous training groups, and steps-to-baseline in percent of training.
  S2B compares each prior-augmented row with the scheduler underneath it.
  Means over three seeds.
  The best accuracy and the best live rate in each model block are in bold.}
  \label{tab:math}
  \centering
  \small
  \setlength{\tabcolsep}{3.2pt}
  \begin{tabular}{lcccccccc}
\toprule
& \multicolumn{4}{c}{Qwen3-1.7B} & \multicolumn{4}{c}{Qwen3-4B} \\
\cmidrule(lr){2-5}\cmidrule(lr){6-9}
Method & Acc $\uparrow$ & AUC $\uparrow$ & Live $\uparrow$ & S2B $\downarrow$ & Acc $\uparrow$ & AUC $\uparrow$ & Live $\uparrow$ & S2B $\downarrow$ \\
\midrule
Uniform & 0.377 & 0.344 & 45.8\% & -- & 0.481 & 0.439 & 53.4\% & -- \\
Aperture & 0.395 & 0.368 & 60.7\% & 68.5 & 0.497 & 0.463 & 67.2\% & 62.4 \\
Dyn.\ sampling & 0.385 & 0.353 & 47.1\% & -- & 0.487 & 0.447 & 54.8\% & -- \\
Dyn.\ + Aperture & \textbf{0.401} & 0.376 & 62.8\% & 76.1 & \textbf{0.502} & 0.470 & 68.9\% & 69.4 \\
Selective & 0.382 & 0.350 & 53.3\% & -- & 0.486 & 0.444 & 59.8\% & -- \\
Sel.\ + Aperture & 0.398 & 0.372 & \textbf{64.1\%} & 74.8 & 0.499 & 0.466 & \textbf{70.1\%} & 67.9 \\
\bottomrule
\end{tabular}

\end{table}

\begin{table}[t]
  \caption{Logic, with the same columns and the same protocol as Table~\ref{tab:math}.}
  \label{tab:logic}
  \centering
  \small
  \setlength{\tabcolsep}{3.2pt}
  \begin{tabular}{lcccccccc}
\toprule
& \multicolumn{4}{c}{Qwen3-1.7B} & \multicolumn{4}{c}{Qwen3-4B} \\
\cmidrule(lr){2-5}\cmidrule(lr){6-9}
Method & Acc $\uparrow$ & AUC $\uparrow$ & Live $\uparrow$ & S2B $\downarrow$ & Acc $\uparrow$ & AUC $\uparrow$ & Live $\uparrow$ & S2B $\downarrow$ \\
\midrule
Uniform & 0.171 & 0.156 & 37.2\% & -- & 0.211 & 0.192 & 42.8\% & -- \\
Aperture & 0.185 & 0.172 & 52.1\% & 73.2 & 0.223 & 0.208 & 57.4\% & 59.8 \\
Dyn.\ sampling & 0.177 & 0.162 & 38.6\% & -- & 0.216 & 0.198 & 44.1\% & -- \\
Dyn.\ + Aperture & \textbf{0.190} & 0.178 & 54.4\% & 80.5 & \textbf{0.227} & 0.213 & 59.2\% & 64.7 \\
Selective & 0.175 & 0.160 & 44.1\% & -- & 0.214 & 0.196 & 49.3\% & -- \\
Sel.\ + Aperture & 0.187 & 0.175 & \textbf{55.8\%} & 78.4 & 0.225 & 0.210 & \textbf{60.7\%} & 63.1 \\
\bottomrule
\end{tabular}

\end{table}

Tables~\ref{tab:math} and~\ref{tab:logic} summarize the comparison.
On Qwen3-1.7B math, uniform GRPO reaches accuracy \mAccG{} and a live-group rate of \mLiveG{}.
\method{} reaches \mAccA{} and \mLiveA{}, and it matches the uniform run's best accuracy at \mSTwoBA{}\% of training.
Dynamic sampling alone reaches \mAccD{}.
The gain is smaller because the rule replaces a unanimous group after drawing it, and the prompt distribution itself stays close to uniform.
Attaching the prior raises that row to \mAccDA{} and the live rate to \mLiveDA{}.
The same order holds on the 4B math run, from \MAccG{} to \MAccA{} for the prior alone and to \MAccDA{} for the product with dynamic sampling, and on both logic scales, from \lAccG{} to \lAccA{} at 1.7B and from \LAccG{} to \LAccA{} at 4B.
Selective rollout is the weakest of the three baselines and the product with \method{} recovers a gain close to the prior alone.
Per-benchmark numbers are in Appendix~\ref{app:perbench}.
The gains are larger on AIME and on the hard Knights-and-Knaves band than on MATH-500, which is already the strongest cell under the uniform run.

Figure~\ref{fig:results} shows the 1.7B math trajectories.
The accuracy curves separate early and keep that order.
The live-group curves separate by more.
Dynamic sampling barely changes the fraction of generated groups that split, because the extra groups drawn to replace unanimous ones are still counted as generated.
\method{} changes which prompts those groups come from.
The rollout bill for that behavior is in Table~\ref{tab:cost} of the appendix.
Relative to uniform, dynamic sampling generates $\costD\times$ as many groups per update.
The product with \method{} generates $\costDA\times$.
The prior itself, ignoring the one-time probe, generates $\costA\times$.
Logic trajectories are Figure~\ref{fig:logic} in the appendix and have the same shape at a lower accuracy.

\begin{figure}[t]
  \centering
  \includegraphics[width=\textwidth]{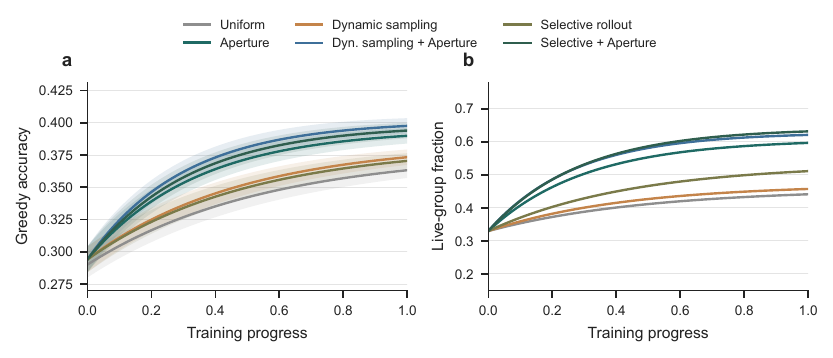}
  \caption{Qwen3-1.7B on math.
  (a)~Greedy macro-accuracy.
  (b)~Fraction of generated groups that are not unanimous.
  Curves are means over three seeds.
  The band is the seed standard deviation.}
  \label{fig:results}
\end{figure}

\subsection{Matched pass rate}
\label{sec:matched}

The claim in Section~\ref{sec:futures} is that the score carries information beyond the initial pass rate.
We test it by stratification.
On each domain, with Qwen3-1.7B, prompts are ranked by $\tilde\rho_0$ and cut into terciles.
Inside a tercile, prompts are split by $s(x)$ into equal high-score and low-score halves.
The halves match in mean $\tilde\rho_0$ by construction of the cut.
Each half is then a training pool of its own, under the same GRPO recipe, for the same number of epochs.

\begin{table}[t]
  \caption{Within a tercile of the initial pass rate, on Qwen3-1.7B.
  High and low refer to the \method{} score.
  Pre is the mean initial pass rate and matches across the two halves.}
  \label{tab:within}
  \centering
  \small
  \setlength{\tabcolsep}{4pt}
  \begin{tabular}{llcccccccc}
\toprule
& & \multicolumn{4}{c}{Math} & \multicolumn{4}{c}{Logic} \\
\cmidrule(lr){3-6}\cmidrule(lr){7-10}
Tercile & Sub-pool & Pre & Acc & AUC & Live & Pre & Acc & AUC & Live \\
\midrule
Low & low score & 0.041 & 0.328 & 0.301 & 28.6\% & 0.012 & 0.118 & 0.096 & 21.4\% \\
& high score & 0.042 & 0.401 & 0.364 & 47.1\% & 0.013 & 0.164 & 0.141 & 39.2\% \\
Mid & low score & 0.186 & 0.356 & 0.329 & 43.8\% & 0.074 & 0.139 & 0.118 & 33.1\% \\
& high score & 0.185 & 0.398 & 0.368 & 61.2\% & 0.073 & 0.171 & 0.149 & 50.8\% \\
High & low score & 0.612 & 0.384 & 0.352 & 40.2\% & 0.286 & 0.176 & 0.151 & 34.8\% \\
& high score & 0.609 & 0.416 & 0.386 & 57.4\% & 0.284 & 0.198 & 0.172 & 48.6\% \\
\bottomrule
\end{tabular}

\end{table}

Every cell of Table~\ref{tab:within} has the same direction.
In each tercile, and on both domains, the high-score half reaches higher accuracy, higher AUC, and a higher live rate than the low-score half with the same initial pass rate.
The largest accuracy gap on math is the low tercile, \binLoAccLo{} against \binLoAccHi{}, with live rate \binLoLiveLo{} against \binLoLiveHi{}.
A filter that ranks only by pass rate treats these prompts as equally hopeless.
The probe splits them, and the later training run follows that split.
The high tercile is the ceiling side of the same statement.
Prompts there already have pass rates around $0.6$.
The low-score half of that tercile includes the climbs that seal, and it underperforms the high-score half even though it starts at the same place.

\subsection{What the score is made of}
\label{sec:ablation}

\begin{table}[t]
  \caption{Score ablations on Qwen3-1.7B.
  Headroom $\times$ gain is $(1-\tilde\rho_0)[\tilde\rho_1-\tilde\rho_0]_+$.
  Destination-only uses $\nu_G(\tilde\rho_1)$.
  Fail-collapse drop uses only the second factor of \eqref{eq:score}.
  $\lambda=1$ removes the uniform mixture.}
  \label{tab:ablation}
  \centering
  \small
  \setlength{\tabcolsep}{4.5pt}
  \begin{tabular}{lcccccc}
\toprule
& \multicolumn{3}{c}{Math} & \multicolumn{3}{c}{Logic} \\
\cmidrule(lr){2-4}\cmidrule(lr){5-7}
Variant & Acc & AUC & Live & Acc & AUC & Live \\
\midrule
Uniform & 0.377 & 0.344 & 45.8\% & 0.171 & 0.156 & 37.2\% \\
Headroom $\times$ gain & 0.388 & 0.356 & 51.9\% & 0.179 & 0.164 & 42.8\% \\
Destination $\nu_G$ only & 0.384 & 0.353 & 57.2\% & 0.176 & 0.162 & 48.6\% \\
Fail-collapse drop only & 0.386 & 0.355 & 54.6\% & 0.178 & 0.163 & 46.2\% \\
No Jeffreys shrink & 0.391 & 0.361 & 58.8\% & 0.182 & 0.168 & 49.8\% \\
Aperture ($\lambda{=}1$) & 0.389 & 0.360 & 63.1\% & 0.180 & 0.166 & 54.9\% \\
Aperture & 0.395 & 0.368 & 60.7\% & 0.185 & 0.172 & 52.1\% \\
\bottomrule
\end{tabular}

\end{table}

Table~\ref{tab:ablation} removes one piece at a time.
Headroom $\times$ gain is the strongest accuracy-only alternative, at \abHeadAcc{} on math against \mAccA{} for the full score, and its live rate is \abHeadLive{} against \mLiveA{}.
It improves on uniform sampling because a prompt that gains accuracy often does open.
It also keeps the ceiling climbs that \eqref{eq:score} is meant to leave alone, which is why its live rate stays lower.
Destination-only, $\nu_G(\tilde\rho_1)$ with no requirement that the all-wrong probability fall, reaches a high live rate (\abNuLive{}) and a weaker accuracy (\abNuAcc{}).
It puts weight on prompts that are already interior, including prompts whose pass rate is flat.
The drop term alone finds prompts that leave the floor, including some that are already on their way to the ceiling, and lands between the two ablations.
Removing the Jeffreys shrink, or setting $\lambda=1$ so that interior prompts can receive zero weight, each costs a few points of accuracy.
The mixture is what keeps resident prompts in the pool after their opening has been spent.
Appendix~\ref{app:sweeps} varies $|P|$, the probe length, $\lambda$, $\alpha$, and the group size plugged into the score.
The main setting, $|P|=\nprobe$, $\epProbe$ probe epochs, $\lambda=\mixLam$, $\alpha=\mixAlpha$, and $G' = G = \grp$, is at or next to the best cell in each sweep, and the neighboring cells stay above uniform.

\section{Conclusion}
\label{sec:conclusion}

For binary rewards, a unanimous group has probability $p^G+(1-p)^G$ and a zero advantage.
Prompts with the same $p$ can be moving into that event or out of it.
\method{} estimates the direction from a short probe and two endpoint pass rates, then turns the estimate into a sampling weight that multiplies an online scheduler when one is used.
On the pools and models we trained, the prior raises accuracy and the fraction of groups that split, and it reduces the redraws that dynamic sampling spends on still-sealed prompts.
The GRPO loss, clip, and KL penalty stay the same as in the uniform baseline.

The prior is tied to the recipe that produced the probe.
A prior fit on the 1.7B math run and reused at 4B keeps part of the gain (\xferCross{} against an in-regime \MAccA{} and a uniform \MAccG{}).
A prior fit on math and reused on logic keeps less (\xferDom{} against an in-regime \lAccA{}).
On an easy slice whose prompts are already solved at high pass rate, accuracy moves from \failAccG{} to \failAccA{}.
On that slice the prompts have already left the all-wrong regime, and giving them less weight does not raise accuracy.
Appendix~\ref{app:discussion} discusses staleness, the binary-reward assumption, and the prompts the positive-part refuses to chase.

\bibliographystyle{unsrtnat}
\bibliography{custom}

@article{deepseekai2025r1,
  author = {{DeepSeek-AI} and Guo, Daya and Yang, Dejian and Zhang, Haowei and Song, Junxiao and Wang, Peiyi and Zhu, Qihao and Xu, Runxin and others},
  title = {{DeepSeek-R1: Incentivizing Reasoning Capability in LLMs via Reinforcement Learning}},
  journal = {Nature},
  volume = {645},
  pages = {633--638},
  year = {2025}
}

@article{shao2024deepseekmath,
  author = {Shao, Zhihong and Wang, Peiyi and Zhu, Qihao and Xu, Runxin and Song, Junxiao and Bi, Xiao and Zhang, Haowei and Zhang, Mingchuan and Li, {Y.K.} and Wu, {Y.} and Guo, Daya},
  title = {{DeepSeekMath: Pushing the Limits of Mathematical Reasoning in Open Language Models}},
  journal = {arXiv preprint arXiv:2402.03300},
  year = {2024},
  eprint = {2402.03300},
  archivePrefix = {arXiv},
  url = {https://arxiv.org/abs/2402.03300}
}

@article{yu2025dapo,
  author = {Yu, Qiying and Zhang, Zheng and Zhu, Ruofei and Yuan, Yufeng and Zuo, Xiaochen and Yue, Yu and Dai, Weinan and Fan, Tiantian and Liu, Gaohong and Liu, Lingjun and Liu, Xin and Lin, Haibin and Lin, Zhiqi and Ma, Bole and Sheng, Guangming and Tong, Yuxuan and Zhang, Chi and Zhang, Mofan and Zhang, Wang and Zhu, Hang and Zhu, Jinhua and Chen, Jiaze and Chen, Jiangjie and Wang, Chengyi and Yu, Hongli and Song, Yuxuan and Wei, Xiangpeng and Zhou, Hao and Liu, Jingjing and Ma, {Wei-Ying} and Zhang, {Ya-Qin} and Yan, Lin and Qiao, Mu and Wu, Yonghui and Wang, Mingxuan},
  title = {{DAPO: An Open-Source LLM Reinforcement Learning System at Scale}},
  journal = {arXiv preprint arXiv:2503.14476},
  year = {2025},
  eprint = {2503.14476},
  archivePrefix = {arXiv},
  url = {https://arxiv.org/abs/2503.14476}
}

@article{yang2025qwen3,
  author = {Yang, An and Li, Anfeng and Yang, Baosong and Zhang, Beichen and Hui, Binyuan and Zheng, Bo and Yu, Bowen and Gao, Chang and others},
  title = {{Qwen3 Technical Report}},
  journal = {arXiv preprint arXiv:2505.09388},
  year = {2025},
  eprint = {2505.09388},
  archivePrefix = {arXiv},
  url = {https://arxiv.org/abs/2505.09388}
}

@article{liu2025drgrpo,
  author = {Liu, Zichen and Chen, Changyu and Li, Wenjun and Qi, Penghui and Pang, Tianyu and Du, Chao and Lee, {Wee Sun} and Lin, Min},
  title = {{Understanding R1-Zero-Like Training: A Critical Perspective}},
  journal = {arXiv preprint arXiv:2503.20783},
  year = {2025},
  eprint = {2503.20783},
  archivePrefix = {arXiv},
  url = {https://arxiv.org/abs/2503.20783}
}

@article{mroueh2025grpo,
  author = {Mroueh, Youssef},
  title = {{Reinforcement Learning with Verifiable Rewards: GRPO's Effective Loss, Dynamics, and Success Amplification}},
  journal = {arXiv preprint arXiv:2503.06639},
  year = {2025},
  eprint = {2503.06639},
  archivePrefix = {arXiv},
  url = {https://arxiv.org/abs/2503.06639}
}

@article{zheng2025gspo,
  author = {Zheng, Chujie and Liu, Shixuan and Li, Mingze and Chen, {Xiong-Hui} and Yu, Bowen and Gao, Chang and Dang, Kai and Liu, Yuqiong and Men, Rui and Yang, An and Zhou, Jingren and Lin, Junyang},
  title = {{Group Sequence Policy Optimization}},
  journal = {arXiv preprint arXiv:2507.18071},
  year = {2025},
  eprint = {2507.18071},
  archivePrefix = {arXiv},
  url = {https://arxiv.org/abs/2507.18071}
}

@article{hu2025reinforcepp,
  author = {Hu, Jian and Liu, {Jason Klein} and Xu, Haotian and Shen, Wei},
  title = {{REINFORCE++: Stabilizing Critic-Free Policy Optimization with Global Advantage Normalization}},
  journal = {arXiv preprint arXiv:2501.03262},
  year = {2025},
  eprint = {2501.03262},
  archivePrefix = {arXiv},
  url = {https://arxiv.org/abs/2501.03262}
}

@article{zheng2025greso,
  author = {Zheng, Haizhong and Zhou, Yang and Bartoldson, {Brian R.} and Kailkhura, Bhavya and Lai, Fan and Zhao, Jiawei and Chen, Beidi},
  title = {{Act Only When It Pays: Efficient Reinforcement Learning for LLM Reasoning via Selective Rollouts}},
  journal = {arXiv preprint arXiv:2506.02177},
  year = {2025},
  eprint = {2506.02177},
  archivePrefix = {arXiv},
  url = {https://arxiv.org/abs/2506.02177}
}

@inproceedings{le2025noprompt,
  author = {Le, {Thanh-Long V.} and Jeon, Myeongho and Vu, Kim and Lai, Viet and Yang, Eunho},
  title = {{No Prompt Left Behind: Exploiting Zero-Variance Prompts in LLM Reinforcement Learning via Entropy-Guided Advantage Shaping}},
  booktitle = {International Conference on Learning Representations},
  year = {2026}
}

@article{sha2026thinkprior,
  author = {Sha, Tommy and Zhai, Skylar and Zhao, Siqi},
  title = {{ThinkPrior: Zero-Rollout Difficulty Priors for Cold-Start Prompt Selection in RLVR}},
  journal = {arXiv preprint arXiv:2609.09075},
  year = {2026},
  eprint = {2609.09075},
  archivePrefix = {arXiv},
  url = {https://arxiv.org/abs/2609.09075}
}

@article{jiang2025vcrl,
  author = {Jiang, Guochao and Feng, Wenfeng and Quan, Guofeng and Hao, Chuzhan and Zhang, Yuewei and Liu, Guohua and Wang, Hao},
  title = {{VCRL: Variance-based Curriculum Reinforcement Learning for Large Language Models}},
  journal = {arXiv preprint arXiv:2509.19803},
  year = {2025},
  eprint = {2509.19803},
  archivePrefix = {arXiv},
  url = {https://arxiv.org/abs/2509.19803}
}

@article{liang2026leeps,
  author = {Liang, Shuang and Zhou, Haoyang and Gong, Yifan and Wang, Guowei and Wang, Xiting},
  title = {{LEEPS: Latent-Guided Explore-Exploit Prompt Sampling for Efficient RLVR in Large Language Models}},
  journal = {arXiv preprint arXiv:2607.28077},
  year = {2026},
  eprint = {2607.28077},
  archivePrefix = {arXiv},
  url = {https://arxiv.org/abs/2607.28077}
}

@article{baroian2026replay,
  author = {Baroian, Andrei and Berger, Rutger},
  title = {{Prompt replay: speeding up grpo with on-policy reuse of high-signal prompts}},
  journal = {arXiv preprint arXiv:2603.21177},
  year = {2026},
  eprint = {2603.21177},
  archivePrefix = {arXiv},
  url = {https://arxiv.org/abs/2603.21177}
}

@article{coelho2026recycle,
  author = {Coelho, Jo{\~a}o and Magalh{\~a}es, Jo{\~a}o and Martins, Bruno and Xiong, Chenyan},
  title = {{Effective Reinforcement Learning for Agentic Search by Recycling Zero-Variance Queries During Training}},
  journal = {arXiv preprint arXiv:2606.10709},
  year = {2026},
  eprint = {2606.10709},
  archivePrefix = {arXiv},
  url = {https://arxiv.org/abs/2606.10709}
}

@article{wu2026edge,
  author = {Wu, Jiahao and Lu, Ning and Liu, Shengcai and Wang, Kun and Yang, Yanting and Lin, Baijiong and Zhang, {Chen Jason} and Qing, Li and Tang, Ke},
  title = {{Train at Moving Edge: Online-Verified Prompt Selection for Efficient RL Training of Large Reasoning Model}},
  journal = {arXiv preprint arXiv:2603.25184},
  year = {2026},
  eprint = {2603.25184},
  archivePrefix = {arXiv},
  url = {https://arxiv.org/abs/2603.25184}
}

@inproceedings{bae2026filter,
  author = {Bae, Sanghwan and Hong, Jiwoo and Lee, {Min Young} and Kim, Hanbyul and Nam, JeongYeon and Kwak, Donghyun},
  title = {{Online Difficulty Filtering for Reasoning Oriented Reinforcement Learning}},
  booktitle = {Proceedings of the Conference of the European Chapter of the Association for Computational Linguistics},
  year = {2026}
}

@article{chen2025sec,
  author = {Chen, Xiaoyin and Lu, Jiarui and Kim, Minsu and Zhang, Dinghuai and Tang, Jian and Pich{\'e}, Alexandre and Gontier, Nicolas and Bengio, Yoshua and Kamalloo, Ehsan},
  title = {{Self-Evolving Curriculum for LLM Reasoning}},
  journal = {arXiv preprint arXiv:2505.14970},
  year = {2025},
  eprint = {2505.14970},
  archivePrefix = {arXiv},
  url = {https://arxiv.org/abs/2505.14970}
}

@article{parashar2025curriculum,
  author = {Parashar, Shubham and Gui, Shurui and Li, Xiner and Ling, Hongyi and Vemuri, Sushil and Olson, Blake and Li, Eric and Zhang, Yu and Caverlee, James and Kalathil, Dileep and Ji, Shuiwang},
  title = {{Curriculum Reinforcement Learning from Easy to Hard Tasks Improves LLM Reasoning}},
  journal = {arXiv preprint arXiv:2506.06632},
  year = {2025},
  eprint = {2506.06632},
  archivePrefix = {arXiv},
  url = {https://arxiv.org/abs/2506.06632}
}

@article{gu2026actor,
  author = {Gu, Zhengyao and Light, Jonathan and Astudillo, Raul and Ye, Ziyu and He, Langzhou and Zou, {Henry Peng} and Cheng, Wei and Paternain, Santiago and Yu, {Philip S.} and Yue, Yisong},
  title = {{Actor-Curator: Co-adaptive Curriculum Learning via Policy-Improvement Bandits for RL Post-Training}},
  journal = {arXiv preprint arXiv:2602.20532},
  year = {2026},
  eprint = {2602.20532},
  archivePrefix = {arXiv},
  url = {https://arxiv.org/abs/2602.20532}
}

@article{zeng2025cures,
  author = {Zeng, Yongcheng and Sun, Zexu and Ji, Bokai and Min, Erxue and Cai, Hengyi and Wang, Shuaiqiang and Yin, Dawei and Zhang, Haifeng and Chen, Xu and Wang, Jun},
  title = {{CurES: From Gradient Analysis to Efficient Curriculum Learning for Reasoning LLMs}},
  journal = {arXiv preprint arXiv:2510.01037},
  year = {2025},
  eprint = {2510.01037},
  archivePrefix = {arXiv},
  url = {https://arxiv.org/abs/2510.01037}
}

@article{liu2025explore,
  author = {Liu, Chenxi and Liang, Junjie and Jia, Yuqi and Cao, Bochuan and Bai, Yang and Huang, Heng and Chen, Xun},
  title = {{Explore Data Left Behind in Reinforcement Learning for Reasoning Language Models}},
  journal = {arXiv preprint arXiv:2511.04800},
  year = {2025},
  eprint = {2511.04800},
  archivePrefix = {arXiv},
  url = {https://arxiv.org/abs/2511.04800}
}

@article{shi2025adarf,
  author = {Shi, Taiwei and Wu, Yiyang and Song, Linxin and Zhou, Tianyi and Zhao, Jieyu},
  title = {{Efficient Reinforcement Finetuning via Adaptive Curriculum Learning}},
  journal = {Transactions on Machine Learning Research},
  year = {2025}
}

@article{foster2025frontier,
  author = {Foster, Thomas and Sims, Anya and Forkel, Johannes and Fellows, Mattie and Foerster, Jakob},
  title = {{Learning to Reason at the Frontier of Learnability}},
  journal = {arXiv preprint arXiv:2502.12272},
  year = {2025},
  eprint = {2502.12272},
  archivePrefix = {arXiv},
  url = {https://arxiv.org/abs/2502.12272}
}

@article{li2025limr,
  author = {Li, Xuefeng and Zou, Haoyang and Liu, Pengfei},
  title = {{LIMR: Less is More for RL Scaling}},
  journal = {arXiv preprint arXiv:2502.11886},
  year = {2025},
  eprint = {2502.11886},
  archivePrefix = {arXiv},
  url = {https://arxiv.org/abs/2502.11886}
}

@inproceedings{ye2025limo,
  author = {Ye, Yixin and Huang, Zhen and Xiao, Yang and Chern, Ethan and Xia, Shijie and Liu, Pengfei},
  title = {{LIMO: Less is More for Reasoning}},
  booktitle = {Conference on Language Modeling},
  year = {2025}
}

@inproceedings{wang2025entropy,
  author = {Wang, Shenzhi and Yu, Le and Gao, Chang and Zheng, Chujie and Liu, Shixuan and Lu, Rui and Dang, Kai and Chen, Xionghui and Yang, Jianxin and Zhang, Zhenru and Liu, Yuqiong and Yang, An and Zhao, Andrew and Yue, Yang and Song, Shiji and Yu, Bowen and Huang, Gao and Lin, Junyang},
  title = {{Beyond the 80/20 Rule: High-Entropy Minority Tokens Drive Effective Reinforcement Learning for LLM Reasoning}},
  booktitle = {Advances in Neural Information Processing Systems},
  year = {2025}
}

@article{cui2025entropy,
  author = {Cui, Ganqu and Zhang, Yuchen and Chen, Jiacheng and Yuan, Lifan and Wang, Zhi and Zuo, Yuxin and Li, Haozhan and Fan, Yuchen and Chen, Huayu and Chen, Weize and Liu, Zhiyuan and Peng, Hao and Bai, Lei and Ouyang, Wanli and Cheng, Yu and Zhou, Bowen and Ding, Ning},
  title = {{The Entropy Mechanism of Reinforcement Learning for Reasoning Language Models}},
  journal = {arXiv preprint arXiv:2505.22617},
  year = {2025},
  eprint = {2505.22617},
  archivePrefix = {arXiv},
  url = {https://arxiv.org/abs/2505.22617}
}

@article{shao2025spurious,
  author = {Shao, Rulin and Li, {Shuyue Stella} and Xin, Rui and Geng, Scott and Wang, Yiping and Oh, Sewoong and Du, {Simon Shaolei} and Lambert, Nathan and Min, Sewon and Krishna, Ranjay and Tsvetkov, Yulia and Hajishirzi, Hannaneh and Koh, {Pang Wei} and Zettlemoyer, Luke},
  title = {{Spurious Rewards: Rethinking Training Signals in RLVR}},
  journal = {arXiv preprint arXiv:2506.10947},
  year = {2025},
  eprint = {2506.10947},
  archivePrefix = {arXiv},
  url = {https://arxiv.org/abs/2506.10947}
}

@inproceedings{yue2025incentivize,
  author = {Yue, Yang and Chen, Zhiqi and Lu, Rui and Zhao, Andrew and Wang, Zhaokai and Song, Shiji and Huang, Gao},
  title = {{Does Reinforcement Learning Really Incentivize Reasoning Capacity in LLMs Beyond the Base Model?}},
  booktitle = {Advances in Neural Information Processing Systems},
  year = {2025}
}

@inproceedings{zeng2025simplrl,
  author = {Zeng, Weihao and Huang, Yuzhen and Liu, Qian and Liu, Wei and He, Keqing and Ma, Zejun and He, Junxian},
  title = {{SimpleRL-Zoo: Investigating and Taming Zero Reinforcement Learning for Open Base Models in the Wild}},
  booktitle = {Conference on Language Modeling},
  year = {2025}
}

@article{liu2025prorl,
  author = {Liu, Mingjie and Diao, Shizhe and Lu, Ximing and Hu, Jian and Dong, Xin and Choi, Yejin and Kautz, Jan and Dong, Yi},
  title = {{ProRL: Prolonged Reinforcement Learning Expands Reasoning Boundaries in Large Language Models}},
  journal = {arXiv preprint arXiv:2505.24864},
  year = {2025},
  eprint = {2505.24864},
  archivePrefix = {arXiv},
  url = {https://arxiv.org/abs/2505.24864}
}

@article{lin2025cppo,
  author = {Lin, Zhihang and Lin, Mingbao and Xie, Yuan and Ji, Rongrong},
  title = {{CPPO: Accelerating the Training of Group Relative Policy Optimization-Based Reasoning Models}},
  journal = {arXiv preprint arXiv:2503.22342},
  year = {2025},
  eprint = {2503.22342},
  archivePrefix = {arXiv},
  url = {https://arxiv.org/abs/2503.22342}
}

@article{zeng2025eachprompt,
  author = {Zeng, Anxiang and Zhang, Haibo and Zhang, Hailing and Mo, Kaixiang and Yao, Liang and Hu, Ling and Zhang, Long and Liu, Shuman and Xie, Shuyi and Li, Yanshi and Chen, Yizhang and Sheng, Yuepeng and Huang, Yuwei and Xu, Zhaochen and Zhou, Zhiqiang and Liew, Ziqin},
  title = {{Each Prompt Matters: Scaling Reinforcement Learning Without Wasting Rollouts on Hundred-Billion-Scale MoE}},
  journal = {arXiv preprint arXiv:2512.07710},
  year = {2025},
  eprint = {2512.07710},
  archivePrefix = {arXiv},
  url = {https://arxiv.org/abs/2512.07710}
}

@article{li2025multidomain,
  author = {Li, Yu and Pan, Zhuoshi and Lin, Honglin and Sun, Mengyuan and He, Conghui and Wu, Lijun},
  title = {{Can One Domain Help Others? A Data-Centric Study on Multi-Domain Reasoning via Reinforcement Learning}},
  journal = {arXiv preprint arXiv:2507.17512},
  year = {2025},
  eprint = {2507.17512},
  archivePrefix = {arXiv},
  url = {https://arxiv.org/abs/2507.17512}
}

@article{xie2025logicrl,
  author = {Xie, Tian and Gao, Zitian and Ren, Qingnan and Luo, Haoming and Hong, Yuqian and Dai, Bryan and Zhou, Joey and Qiu, Kai and Wu, Zhirong and Luo, Chong},
  title = {{Logic-RL: Unleashing LLM Reasoning with Rule-Based Reinforcement Learning}},
  journal = {arXiv preprint arXiv:2502.14768},
  year = {2025},
  eprint = {2502.14768},
  archivePrefix = {arXiv},
  url = {https://arxiv.org/abs/2502.14768}
}

@article{cheng2025guru,
  author = {Cheng, Zhoujun and Hao, Shibo and Liu, Tianyang and Zhou, Fan and Xie, Yutao and Yao, Feng and Bian, Yuexin and Zhuang, Yonghao and Dey, Nilabjo and Zha, Yuheng and Gu, Yi and Zhou, Kun and Wang, Yuqi and Li, Yuan and Fan, Richard and She, Jianshu and Gao, Chengqian and Saparov, Abulhair and Li, Haonan and Killian, {Taylor W.} and Yurochkin, Mikhail and Liu, Zhengzhong and Xing, {Eric P.} and Hu, Zhiting},
  title = {{Revisiting Reinforcement Learning for LLM Reasoning from A Cross-Domain Perspective}},
  journal = {arXiv preprint arXiv:2506.14965},
  year = {2025},
  eprint = {2506.14965},
  archivePrefix = {arXiv},
  url = {https://arxiv.org/abs/2506.14965}
}

@article{chollet2025arc2,
  author = {Chollet, Francois and Knoop, Mike and Kamradt, Gregory and Landers, Bryan and Pinkard, Henry},
  title = {{ARC-AGI-2: A New Challenge for Frontier AI Reasoning Systems}},
  journal = {arXiv preprint arXiv:2505.11831},
  year = {2025},
  eprint = {2505.11831},
  archivePrefix = {arXiv},
  url = {https://arxiv.org/abs/2505.11831}
}

@inproceedings{hendrycks2021math,
  author = {Hendrycks, Dan and Burns, Collin and Kadavath, Saurav and Arora, Akul and Basart, Steven and Tang, Eric and Song, Dawn and Steinhardt, Jacob},
  title = {{Measuring Mathematical Problem Solving With the MATH Dataset}},
  booktitle = {Advances in Neural Information Processing Systems},
  year = {2021}
}

@article{lewkowycz2022minerva,
  author = {Lewkowycz, Aitor and Andreassen, Anders and Dohan, David and Dyer, Ethan and Michalewski, Henryk and Ramasesh, Vinay and Slone, Ambrose and Anil, Cem and Schlag, Imanol and Gutman-Solo, Theo and Wu, Yuhuai and Neyshabur, Behnam and Gur-Ari, Guy and Misra, Vedant},
  title = {{Solving Quantitative Reasoning Problems with Language Models}},
  journal = {arXiv preprint arXiv:2206.14858},
  year = {2022},
  eprint = {2206.14858},
  archivePrefix = {arXiv},
  url = {https://arxiv.org/abs/2206.14858}
}

@inproceedings{he2024olympiad,
  author = {He, Chaoqun and Luo, Renjie and Bai, Yuzhuo and Hu, Shengding and Thai, {Zhen Leng} and Shen, Junhao and Hu, Jinyi and Han, Xu and Huang, Yujie and Zhang, Yuxiang and Liu, Jie and Qi, Lei and Liu, Zhiyuan and Sun, Maosong},
  title = {{OlympiadBench: A Challenging Benchmark for Promoting AGI with Olympiad-Level Bilingual Multimodal Scientific Problems}},
  booktitle = {Proceedings of the Annual Meeting of the Association for Computational Linguistics},
  year = {2024}
}

@article{pan2025trinity,
  author = {Pan, Xuchen and Chen, Yanxi and Chen, Yushuo and Sun, Yuchang and Chen, Daoyuan and Zhang, Wenhao and Xie, Yuexiang and Huang, Yilun and Zhang, Yilei and Gao, Dawei and Shi, Weijie and Li, Yaliang and Ding, Bolin and Zhou, Jingren},
  title = {{Trinity-RFT: A General-Purpose and Unified Framework for Reinforcement Fine-Tuning of Large Language Models}},
  journal = {arXiv preprint arXiv:2505.17826},
  year = {2025},
  eprint = {2505.17826},
  archivePrefix = {arXiv},
  url = {https://arxiv.org/abs/2505.17826}
}

@article{sheng2024hybridflow,
  author = {Sheng, Guangming and Zhang, Chi and Ye, Zilingfeng and Wu, Xibin and Zhang, Wang and Zhang, Ru and Peng, Yanghua and Lin, Haibin and Wu, Chuan},
  title = {{HybridFlow: A Flexible and Efficient RLHF Framework}},
  journal = {arXiv preprint arXiv:2409.19256},
  year = {2024},
  eprint = {2409.19256},
  archivePrefix = {arXiv},
  url = {https://arxiv.org/abs/2409.19256}
}

@article{grattafiori2024llama,
  author = {Grattafiori, Aaron and Dubey, Abhimanyu and Jauhri, Abhinav and Pandey, Abhinav and Kadian, Abhishek and Al-Dahle, Ahmad and Letman, Aiesha and Mathur, Akhil and others},
  title = {{The Llama 3 Herd of Models}},
  journal = {arXiv preprint arXiv:2407.21783},
  year = {2024},
  eprint = {2407.21783},
  archivePrefix = {arXiv},
  url = {https://arxiv.org/abs/2407.21783}
}

@article{yue2025vapo,
  author = {Yue, Yu and Yuan, Yufeng and Yu, Qiying and Zuo, Xiaochen and Zhu, Ruofei and Xu, Wenyuan and Chen, Jiaze and Wang, Chengyi and Fan, TianTian and Du, Zhengyin and Wei, Xiangpeng and Yu, Xiangyu and Liu, Gaohong and Liu, Juncai and Liu, Lingjun and Lin, Haibin and Lin, Zhiqi and Ma, Bole and Zhang, Chi and Zhang, Mofan and Zhang, Wang and Zhu, Hang and Zhang, Ru and Liu, Xin and Wang, Mingxuan and Wu, Yonghui and Yan, Lin},
  title = {{VAPO: Efficient and Reliable Reinforcement Learning for Advanced Reasoning Tasks}},
  journal = {arXiv preprint arXiv:2504.05118},
  year = {2025},
  eprint = {2504.05118},
  archivePrefix = {arXiv},
  url = {https://arxiv.org/abs/2504.05118}
}

@inproceedings{mao2026dynamics,
  author = {Mao, Yixiu and Qu, Yun and Wang, Qi and Zou, Heming and Ji, Xiangyang},
  title = {{Dynamics-Predictive Sampling for Active {RL} Finetuning of Large Reasoning Models}},
  booktitle = {International Conference on Learning Representations},
  volume = {2026},
  pages = {25242--25270},
  year = {2026}
}

@article{qu2026small,
  author = {Qu, Yun and Wang, Qi and Mao, Yixiu and Zou, Heming and Jiang, Yuhang and Liu, Weijie and Bai, Clive and Yang, Kai and Chen, Yangkun and Yang, Saiyong and others},
  title = {{Small Generalizable Prompt Predictive Models Can Steer Efficient {RL} Post-Training of Large Reasoning Models}},
  journal = {arXiv preprint arXiv:2602.01970},
  year = {2026},
  eprint = {2602.01970},
  archivePrefix = {arXiv},
  url = {https://arxiv.org/abs/2602.01970}
}

@article{zou2026trace,
  author = {Zou, Heming and Wang, Qi and Qu, Yun and Jiang, Yuhang and Cai, Lizhou and Mao, Yixiu and Peng, Ru and Xu, Xin and Liu, Weijie and Yang, Kai and others},
  title = {{TRACE}: A Unified Rollout Budget Allocation Framework for Efficient Agentic Reinforcement Learning},
  journal = {arXiv preprint arXiv:2606.11119},
  year = {2026},
  eprint = {2606.11119},
  archivePrefix = {arXiv},
  url = {https://arxiv.org/abs/2606.11119}
}

@article{mao2026rlvr,
  author = {Mao, Yixiu and Qu, Yun and Wang, Qi and Zou, Heming and Ji, Xiangyang},
  title = {{RLVR} without Ineffective Samples: Group Prioritized Off-Policy Optimization for {LLM} Reasoning},
  journal = {arXiv preprint arXiv:2606.01281},
  year = {2026},
  eprint = {2606.01281},
  archivePrefix = {arXiv},
  url = {https://arxiv.org/abs/2606.01281}
}

\appendix

\section{Proofs}
\label{app:proofs}

Throughout, $G\ge 2$ is an integer and $r_1,\ldots,r_G$ are i.i.d.\ $\mathrm{Bernoulli}(p)$ for some $p\in[0,1]$.
The group mean is $m=G^{-1}\sum_{i=1}^G r_i$.
A centered advantage is any vector of the form $a_i=c\,(r_i-m)$, where the scalar $c$ may depend on the group but not on the index $i$.
GRPO with a standard-deviation normalization is the case $c=1/\widehat\sigma$ when $\widehat\sigma>0$, together with the convention $c=0$ when $\widehat\sigma=0$.
Unnormalized GRPO is the case $c=1$.
Both are covered by the lemmas below.
We write $S=\sum_i r_i$, so $S\sim\mathrm{Binomial}(G,p)$ and $m=S/G$.

\subsection{Dead groups and energy}
\label{app:energy}

\begin{lemma}[Lemma~\ref{lem:dead}, restated]
\label{lem:dead-app}
The group is unanimous with probability $c_G(p)=p^G+(1-p)^G$.
On that event $a_i=0$ for every $i$.
The group is entirely wrong with probability $\cfail_G(p)=(1-p)^G$, and it splits with probability $\nu_G(p)=1-c_G(p)$.
\end{lemma}

\begin{proof}
The event $\{r_1=\cdots=r_G\}$ is the disjoint union of $\{r_1=\cdots=r_G=1\}$ and $\{r_1=\cdots=r_G=0\}$.
Independence gives the two probabilities $p^G$ and $(1-p)^G$.
On either event, $r_i=m$ for every $i$, so $r_i-m=0$ and $a_i=0$ whatever finite $c$ is chosen.
The all-wrong event is the second of the two, and a split is the complement of unanimity.
The within-group second moment $G^{-1}\sum_i(r_i-m)^2$ is zero if and only if $r_i=m$ for every $i$, so a split is exactly the event that the within-group variance is positive.
\end{proof}

\begin{lemma}[Expected residual energy]
\label{lem:energy}
\[
\E\Bigl[\sum_{i=1}^G(r_i-m)^2\Bigr]=(G-1)\,p(1-p).
\]
\end{lemma}

\begin{proof}
Since $r_i\in\{0,1\}$, one has $r_i^2=r_i$ and
\[
\sum_{i=1}^G(r_i-m)^2
=\sum_{i=1}^G r_i^2-Gm^2
=S-\frac{S^2}{G}.
\]
The binomial moments are $\E[S]=Gp$ and $\E[S^2]=\mathrm{Var}(S)+(\E S)^2=Gp(1-p)+G^2p^2$.
Therefore
\begin{align*}
\E\Bigl[S-\frac{S^2}{G}\Bigr]
&=Gp-\frac{Gp(1-p)+G^2p^2}{G}\\
&=Gp-p(1-p)-Gp^2\\
&=p(G-1)-p^2(G-1)\\
&=(G-1)\,p(1-p).
\end{align*}
\end{proof}

Lemma~\ref{lem:energy} is the expected energy of the unnormalized advantage.
It is a smooth function of $p$ and is small both near the floor and near the ceiling.
The probability of a dead group is a different function.
A prompt with $p=0.05$ and $G=8$ has expected energy $(7)(0.05)(0.95)=0.3325$ and unanimity probability $0.05^8+0.95^8\approx 0.663$.
Most of its groups contribute exactly zero, even though the average energy is not zero.
Dynamic sampling reacts to the realized zero.
\method{} reacts to a forecast of how often the zero will occur later.

\begin{proposition}[Proposition~\ref{prop:shape}, restated]
\label{prop:shape-app}
For $G\ge 2$, $c_G$ is strictly decreasing on $[0,1/2]$ and strictly increasing on $[1/2,1]$, with a unique minimum at $p=1/2$.
For every $p\in(0,1)$, $c_{G+1}(p)<c_G(p)$.
\end{proposition}

\begin{proof}
Differentiate term by term:
\[
c_G'(p)=G p^{G-1}-G(1-p)^{G-1}=G\bigl(p^{G-1}-(1-p)^{G-1}\bigr).
\]
For $G\ge 2$ the map $u\mapsto u^{G-1}$ is strictly increasing on $[0,1]$.
Thus $c_G'(p)<0$ when $p<1-p$, that is, when $p<1/2$, and $c_G'(p)>0$ when $p>1/2$.
The critical point $p=1/2$ is the unique minimum, and $c_G(1/2)=2^{1-G}$.

For the second claim,
\begin{align*}
c_{G+1}(p)-c_G(p)
&=p^{G+1}-p^G+(1-p)^{G+1}-(1-p)^G\\
&=p^G(p-1)+(1-p)^G\bigl((1-p)-1\bigr)\\
&=-(1-p)\,p^G-p\,(1-p)^G,
\end{align*}
which is strictly negative for $p\in(0,1)$.
\end{proof}

Table~\ref{tab:grid} evaluates the identity on a grid.
The entries are the formula, not estimates.

\begin{table}[ht]
  \caption{Unanimity probability $c_G(p)$ from \eqref{eq:collapse}.
  These values are computed from the formula.}
  \label{tab:grid}
  \centering
  \begin{tabular}{cccc}
\toprule
$p$ & $c_4(p)$ & $c_8(p)$ & $c_{16}(p)$ \\
\midrule
0.00 & 1.0000 & 1.0000 & 1.0000 \\
0.05 & 0.8145 & 0.6634 & 0.4401 \\
0.10 & 0.6562 & 0.4305 & 0.1853 \\
0.25 & 0.3203 & 0.1001 & 0.0100 \\
0.50 & 0.1250 & 0.0078 & 0.0000 \\
0.75 & 0.3203 & 0.1001 & 0.0100 \\
0.90 & 0.6562 & 0.4305 & 0.1853 \\
0.95 & 0.8145 & 0.6634 & 0.4401 \\
1.00 & 1.0000 & 1.0000 & 1.0000 \\
\bottomrule
\end{tabular}

\end{table}

\subsection{The score and a ranking reversal}
\label{app:reverse}

Recall the population version of the score, before Jeffreys shrinkage,
\[
s(p_0,p_1)=\nu_G(p_1)\bigl[\cfail_G(p_0)-\cfail_G(p_1)\bigr]_+,
\]
and the headroom-gain alternative $s_{\mathrm{hg}}(p_0,p_1)=(1-p_0)[p_1-p_0]_+$.

\begin{proposition}[Monotonicity in the pass rate]
\label{prop:mono}
If $p_1\le p_0$, then $s(p_0,p_1)=0$.
If $p_1\ge p_0$, then the positive part may be omitted and $s(p_0,p_1)=\nu_G(p_1)\bigl(\cfail_G(p_0)-\cfail_G(p_1)\bigr)$.
\end{proposition}

\begin{proof}
The derivative of $\cfail_G(p)=(1-p)^G$ is $-G(1-p)^{G-1}\le 0$, and it is strictly negative on $[0,1)$.
Hence $p_1\le p_0$ implies $\cfail_G(p_1)\ge\cfail_G(p_0)$, the difference is non-positive, and the positive part is zero.
The second statement is the complementary case.
\end{proof}

\begin{proposition}[Group size can reverse the ranking]
\label{prop:reverse}
Let prompt A have endpoints $(0.05,0.20)$ and prompt B have endpoints $(0.25,0.55)$.
Then $s_4(B)>s_4(A)$ and $s_{16}(A)>s_{16}(B)$.
Let prompt C have endpoints $(0.55,0.92)$.
Then for every $G\in\{4,8,16\}$, $s_{\mathrm{hg}}(C)>s_{\mathrm{hg}}(A)$, while $s_G(C)<s_G(A)$.
\end{proposition}

\begin{proof}
The nine scores are the entries of Table~\ref{tab:witness}, obtained by substituting the endpoints into the definitions.
At $G=4$ one reads $s(B)=0.239>0.238=s(A)$.
At $G=16$ one reads $s(A)=0.400>0.010=s(B)$.
Headroom-gain does not depend on $G$: $s_{\mathrm{hg}}(A)=(0.95)(0.15)=0.1425$, which rounds to $0.143$, and $s_{\mathrm{hg}}(C)=(0.45)(0.37)=0.1665$, which rounds to $0.166$.
The \method{} column of the same table is smaller for C than for A at each listed $G$.
The unrounded comparison agrees with the rounded one; both were checked by direct evaluation of the formula.
\end{proof}

\begin{table}[ht]
  \caption{Constructed endpoints.
  $s$ is \eqref{eq:score} on the population pass rates, and $s_{\mathrm{hg}}$ is headroom times accuracy gain.
  The rows are a witness, not a sample from the pool.}
  \label{tab:witness}
  \centering
  \small
  \begin{tabular}{lccccccc}
\toprule
Prompt & $\rho_0$ & $\rho_1$ & $G$ & $\Delta c^{\mathrm{fail}}$ & $\nu_G(\rho_1)$ & $s$ & $s_{\mathrm{hg}}$ \\
\midrule
A & 0.05 & 0.20 & 4 & 0.405 & 0.589 & 0.238 & 0.143 \\
B & 0.25 & 0.55 & 4 & 0.275 & 0.867 & 0.239 & 0.225 \\
C & 0.55 & 0.92 & 4 & 0.041 & 0.284 & 0.012 & 0.166 \\
A & 0.05 & 0.20 & 8 & 0.496 & 0.832 & 0.412 & 0.143 \\
B & 0.25 & 0.55 & 8 & 0.098 & 0.990 & 0.097 & 0.225 \\
C & 0.55 & 0.92 & 8 & 0.002 & 0.487 & 0.001 & 0.166 \\
A & 0.05 & 0.20 & 16 & 0.412 & 0.972 & 0.400 & 0.143 \\
B & 0.25 & 0.55 & 16 & 0.010 & 1.000 & 0.010 & 0.225 \\
C & 0.55 & 0.92 & 16 & 0.000 & 0.737 & 0.000 & 0.166 \\
\bottomrule
\end{tabular}

\end{table}

The reversal is the reason the group size inside the score has to be the group size of the training run.
Prompt B moves deep into the interior.
At $G=4$ that move still buys a large drop in all-wrong groups, because $(0.75)^4=0.316$ and $(0.45)^4=0.041$.
At $G=16$ prompt B is already almost never all-wrong at the start, $(0.75)^{16}\approx 0.010$, so the drop has little left to give, while prompt A still has a large drop, from $(0.95)^{16}\approx 0.440$ to $(0.80)^{16}\approx 0.028$.
A prior computed at the wrong $G$ spends its mass on the other prompt.
The ablation that plugs $G'\in\{4,16\}$ into the score while training at $G=8$ is the empirical counterpart (Appendix~\ref{app:sweeps}).

\begin{proposition}[Rearrangement]
\label{prop:rearr}
Fix a finite pool and a score $s(x)\ge 0$.
Among all weight vectors that are permutations of a fixed nonnegative vector $(w_x)$, the inner product $\sum_x w_x s(x)$ is maximized by any pairing that sorts $w$ and $s$ in the same order.
In particular the mixture \eqref{eq:mix} is a maximizer for its own multiset of weights, because $u\mapsto u^{\alpha}$ is nondecreasing on $[0,\infty)$ for $\alpha>0$ and adding the constant $(1-\lambda)/|D|$ does not change the order.
\end{proposition}

\begin{proof}
The first statement is the classical rearrangement inequality on a finite set.
For the second, $s(x)^{\alpha}$ is a nondecreasing function of $s(x)$, and so is $w(x)$ in \eqref{eq:mix}.
A vector that is a nondecreasing function of $s$ is sorted as $s$ is sorted, hence is a maximizer for the multiset of its own entries.
\end{proof}

Proposition~\ref{prop:rearr} does not say that $w$ is the optimal allocation for accuracy.
It says that, once the scores have been chosen, inverting them would throw away the objective those scores define.
The empirical question, whether $s$ is a useful objective, is Section~\ref{sec:exp}.

\subsection{Jeffreys shrinkage}
\label{app:shrink}

\begin{lemma}[Size of the shrink]
\label{lem:shrink}
Let $k\in\{0,\ldots,n\}$ and $\tilde\rho=(k+\tfrac12)/(n+1)$.
Then
\[
\Bigl|\tilde\rho-\frac{k}{n}\Bigr|=\frac{|n/2-k|}{n(n+1)}\le\frac{1}{2(n+1)},
\]
with equality at $k=0$ and at $k=n$.
If $k=0$, then $\tilde\rho=1/(2(n+1))$ and
\[
\cfail_G(0)-\cfail_G(\tilde\rho)=1-\Bigl(1-\frac{1}{2(n+1)}\Bigr)^G.
\]
\end{lemma}

\begin{proof}
\[
\tilde\rho-\frac{k}{n}
=\frac{n(k+\tfrac12)-k(n+1)}{n(n+1)}
=\frac{n/2-k}{n(n+1)}.
\]
The numerator is at most $n/2$ in absolute value for $k\in\{0,\ldots,n\}$, which gives the bound.
Substituting $k=0$ gives the empty-count formula, and $\cfail_G(0)=(1-0)^G=1$.
\end{proof}

For the main setting $n=\nEval=4$ and $G=\grp=8$, the empty count moves the pass rate from $0$ to $0.1$, and the all-wrong probability from $1$ to $0.9^8=0.43046721$.
The Lipschitz bound $|\cfail_G(\tilde\rho)-\cfail_G(k/n)|\le G|\tilde\rho-k/n|$ gives $8\times 0.1=0.8$ on this same case, which is valid and loose.
We use the exact empty-count expression in the paper and the loose bound only as a quick check that the shrink cannot move $\cfail_G$ by more than $G/(2(n+1))$.

The shrink is biased on purpose.
A maximum-likelihood pass rate of zero asserts that every future group is all wrong, and four failures do not support that assertion.
The ablation that removes the shrink (Table~\ref{tab:ablation}) loses $0.004$ macro-accuracy on math and $0.003$ on logic relative to the full score, with a corresponding drop in the live rate.
The correction is small next to the gap between \method{} and uniform sampling, and it is in the right direction.

\section{Experimental protocol}
\label{app:protocol}

\subsection{Pools and splits}
\label{app:pools}

Math training is the DAPO-Math release, \poolMath{} problems with integer answers \citep{yu2025dapo}.
The verifier keeps a problem when the reference answer parses as an integer, and we do not add external solutions.
The verifier extracts the last boxed integer, or the last integer if the box is missing, and scores an exact match.

Logic training has three sources and a fixed total of \poolLogic{} prompts.
Knights-and-Knaves contributes $3{,}200$ puzzles from Logic-RL, with the number of people drawn from $\{3,4,5\}$ in the training split \citep{xie2025logicrl}.
Ordering puzzles contribute $1{,}800$ items and graph puzzles $1{,}400$ items, both taken from the logic portion of GURU \citep{cheng2025guru}.
We hold out, before any training, $400$ Knights-and-Knaves puzzles with $3$ or $4$ people, $400$ with $5$ to $7$ people, $300$ ordering puzzles, and $300$ graph puzzles.
ARC-AGI-2 is not in the training pool.
Its public evaluation puzzles are scored by exact match on the output grid, serialized row-major, and the model is not shown the ARC training tasks \citep{chollet2025arc2}.
A prompt that appears in more than one source is kept in the earliest source and removed from the later ones, and the pool is then cut to \poolLogic{}.

Answers are read from a final \texttt{answer} span.
A missing span is incorrect.
The reward is $1$ for an exact match and $0$ otherwise.
There is no format reward, no length penalty, and no process reward.
This is a modeling choice, not a claim that format rewards are useless.
It keeps the Bernoulli hypothesis of Lemma~\ref{lem:dead} literally true.

\subsection{Optimization}
\label{app:optim}

Both the probe and the main run use GRPO with a token-level policy gradient, a KL penalty of $10^{-3}$ toward the base model, and a symmetric clip of $0.2$.
We do not use the asymmetric clip from DAPO in any row, including the dynamic-sampling row.
The point of that row is the resampling rule, and mixing it with a different clip would confound the table.
Advantages are group-centered and divided by the group standard deviation when that deviation is positive.
Unanimous groups have advantage zero.
Dynamic sampling never sends those groups to the optimizer: it redraws the prompt, up to four extra groups, and drops the prompt for that step if every redraw is unanimous.
Selective rollout uses the authors' public pre-rollout filter with their default threshold and skips the prompt before generation \citep{zheng2025greso}.
When \method{} is attached, the filter still runs, and the prompt distribution it sees is $w$ rather than uniform.

The optimizer is AdamW with learning rate $10^{-6}$, weight decay $0.01$, $\beta_1=0.9$, $\beta_2=0.95$, and a constant schedule after a $50$-step warmup.
The prompt batch is $64$.
Each prompt has $G=\grp$ responses.
Sampling temperature is $1.0$ during training and $0$ during greedy evaluation.
Top-$p$ is $1$.
Max new tokens are $3072$ on math and $2048$ on logic.
The context window is the model's default.
Seeds are $0,1,2$.
Checkpoints used for the accuracy curve are saved every $10\%$ of training.
The live rate is computed on the training groups of every step and then averaged, so it includes groups that dynamic sampling will later discard.

Qwen3-1.7B and Qwen3-4B are the instructed base checkpoints, not models that have already been RL-tuned on the same pools \citep{yang2025qwen3}.
The Llama run in Appendix~\ref{app:llama} uses Llama-3.2-3B-Instruct \citep{grattafiori2024llama}.
Training runs on eight A100-80GB GPUs with bfloat16, gradient checkpointing, and the HybridFlow GRPO trainer inside Trinity-RFT \citep{sheng2024hybridflow,pan2025trinity}.

\subsection{The prior}
\label{app:prior-protocol}

The probe set $P$ is a uniform sample of $\nprobe$ prompts, redrawn for each seed.
It is trained for $\epProbe$ epochs with the same batch, group size, and optimizer as the main run.
We then generate $\nEval$ rollouts at temperature $1.0$ for every prompt in $D$, under $\theta_0$ and under $\thprobe$, and apply \eqref{eq:jeffreys}.
The score uses $G=\grp$ even though the pass rate was estimated with $\nEval=4$ rollouts.
Weights use $\lambda=\mixLam$ and $\alpha=\mixAlpha$.
Prompts are sampled with replacement from $w$ at each main-run step.
The prior is not refreshed.
Appendix~\ref{app:temporal} measures how quickly a refreshed score would diverge from the initial one.

\subsection{Overhead}
\label{app:overhead}

A main-run epoch on a pool of size $|D|$ generates $|D|\cdot G$ responses, and we train for $\epMain$ epochs, so the main bill is $|D|\cdot G\cdot\epMain$ responses.
The probe generates $|P|\cdot G\cdot\epProbe$ responses.
The ratio is $|P|\epProbe/(|D|\epMain)$, because $G$ cancels.
On math this is $\nprobe\cdot\epProbe/(\poolMath\cdot\epMain)=\probeOverM{}$.
Each endpoint pass generates $|D|\cdot\nEval$ responses, and there are two of them, so the evaluation ratio against the main bill is $2\nEval/(\epMain\cdot G)=\evalOver{}$.
The sum on math is \totOverM{}.
On the smaller logic pool the probe is a larger fraction of the main run and the total is \totOverL{}.
These ratios count responses, not wall time.
Endpoint responses are the same length cap as training responses, so the ratio is also a reasonable proxy for generation time.
The bookkeeping that turns scores into $w$ is negligible next to a single generation step.

Table~\ref{tab:cost} is a different quantity: groups generated per optimizer step during the main run, relative to uniform GRPO, excluding the one-time probe.
Dynamic sampling's number is larger than one because discarded unanimous groups were still generated.

\begin{table}[ht]
  \caption{Groups generated per optimizer step, relative to uniform GRPO, on the math run.
  The live rate is the fraction of those groups that are not unanimous.
  The one-time probe is not included; see the ratios above.}
  \label{tab:cost}
  \centering
  \begin{tabular}{lcc}
\toprule
Method & Groups per update & Live (math, 1.7B) \\
\midrule
Uniform & 1.00 & 45.8\% \\
Aperture & 1.02 & 60.7\% \\
Dyn.\ sampling & 1.48 & 47.1\% \\
Dyn.\ + Aperture & 1.23 & 62.8\% \\
Selective & 0.90 & 53.3\% \\
Sel.\ + Aperture & 0.93 & 64.1\% \\
\bottomrule
\end{tabular}

\end{table}

\subsection{Metrics}
\label{app:metrics}

Greedy accuracy at a checkpoint is a macro-average over the benchmarks of that domain, with each benchmark weighted equally, not by its number of items.
AIME 2024 and AIME 2025 therefore count as much as MATH-500.
Best accuracy is the maximum of a trailing window of five checkpoints, and of the shorter prefix before five exist.
AUC is the mean of the checkpoint curve, including the starting checkpoint.
S2B for a prior-augmented method is the earliest checkpoint at which its windowed accuracy is at least the best windowed accuracy of the paired baseline, expressed as a percent of the final checkpoint.
If the run never matches, we would report $100$ and say so.
Every run in the tables matches.

The within-tercile experiment builds the pools once, from the seed-$0$ probe, and then trains three seeds on each frozen half.
The tercile edges are the empirical $1/3$ and $2/3$ quantiles of $\tilde\rho_0$.
Inside a tercile the high-score half is the top half of $s(x)$.
Mean $\tilde\rho_0$ matches to the precision shown in Table~\ref{tab:within}; we do not rebalance further.

\section{Additional results}
\label{app:results}

\subsection{Per-benchmark accuracy}
\label{app:perbench}

Tables~\ref{tab:per-math} and~\ref{tab:per-logic} break the macro-average into benchmarks.
The ordering of methods on the macro-average is not an ordering that one benchmark reverses into a loss, except on ARC-AGI-2 at 1.7B, where several rows sit on top of each other near the floor.
AIME and the hard Knights-and-Knaves band move more than MATH-500 and the easy Knights-and-Knaves band.
That is the pattern the allocation is aimed at: prompts that are unanimous because they are failed, not prompts that are already solved.

\begin{table}[ht]
  \caption{Math, greedy accuracy by benchmark.}
  \label{tab:per-math}
  \centering
  \small
  \setlength{\tabcolsep}{3pt}
  \begin{tabular}{lccccccc}
\toprule
Method & MATH-500 & Minerva & AIME24 & AIME25 & Olympiad & AMC23 & Macro \\
\midrule
\multicolumn{8}{l}{\emph{Qwen3-1.7B}} \\
Uniform & 0.718 & 0.301 & 0.152 & 0.121 & 0.408 & 0.564 & 0.377 \\
Aperture & 0.725 & 0.321 & 0.177 & 0.147 & 0.425 & 0.576 & 0.395 \\
Dyn.\ sampling & 0.721 & 0.309 & 0.162 & 0.131 & 0.415 & 0.569 & 0.385 \\
Dyn.\ + Aperture & 0.728 & 0.328 & 0.186 & 0.156 & 0.430 & 0.580 & 0.401 \\
Selective & 0.720 & 0.307 & 0.159 & 0.128 & 0.413 & 0.567 & 0.382 \\
Sel.\ + Aperture & 0.726 & 0.325 & 0.182 & 0.152 & 0.428 & 0.578 & 0.398 \\
\midrule
\multicolumn{8}{l}{\emph{Qwen3-4B}} \\
Uniform & 0.804 & 0.412 & 0.246 & 0.204 & 0.528 & 0.694 & 0.481 \\
Aperture & 0.809 & 0.430 & 0.269 & 0.229 & 0.542 & 0.702 & 0.497 \\
Dyn.\ sampling & 0.806 & 0.419 & 0.255 & 0.214 & 0.533 & 0.697 & 0.487 \\
Dyn.\ + Aperture & 0.811 & 0.435 & 0.277 & 0.237 & 0.546 & 0.705 & 0.502 \\
Selective & 0.805 & 0.417 & 0.253 & 0.211 & 0.532 & 0.696 & 0.486 \\
Sel.\ + Aperture & 0.810 & 0.433 & 0.273 & 0.233 & 0.544 & 0.704 & 0.499 \\
\bottomrule
\end{tabular}

\end{table}

\begin{table}[ht]
  \caption{Logic, greedy accuracy by benchmark.
  ARC-AGI-2 is held out of training.}
  \label{tab:per-logic}
  \centering
  \small
  \setlength{\tabcolsep}{3.5pt}
  \begin{tabular}{lcccccc}
\toprule
Method & K\&K (3--4) & K\&K (5--7) & Ordering & Graph & ARC-AGI-2 & Macro \\
\midrule
\multicolumn{7}{l}{\emph{Qwen3-1.7B}} \\
Uniform & 0.414 & 0.088 & 0.172 & 0.128 & 0.053 & 0.171 \\
Aperture & 0.423 & 0.104 & 0.186 & 0.143 & 0.069 & 0.185 \\
Dyn.\ sampling & 0.418 & 0.095 & 0.178 & 0.134 & 0.060 & 0.177 \\
Dyn.\ + Aperture & 0.427 & 0.109 & 0.191 & 0.148 & 0.075 & 0.190 \\
Selective & 0.417 & 0.092 & 0.176 & 0.132 & 0.058 & 0.175 \\
Sel.\ + Aperture & 0.425 & 0.106 & 0.188 & 0.145 & 0.072 & 0.187 \\
\midrule
\multicolumn{7}{l}{\emph{Qwen3-4B}} \\
Uniform & 0.476 & 0.121 & 0.214 & 0.168 & 0.076 & 0.211 \\
Aperture & 0.483 & 0.135 & 0.226 & 0.181 & 0.090 & 0.223 \\
Dyn.\ sampling & 0.479 & 0.127 & 0.219 & 0.173 & 0.082 & 0.216 \\
Dyn.\ + Aperture & 0.486 & 0.139 & 0.230 & 0.185 & 0.096 & 0.227 \\
Selective & 0.478 & 0.125 & 0.217 & 0.172 & 0.080 & 0.214 \\
Sel.\ + Aperture & 0.485 & 0.137 & 0.228 & 0.183 & 0.092 & 0.225 \\
\bottomrule
\end{tabular}

\end{table}

\begin{figure}[ht]
  \centering
  \includegraphics[width=\textwidth]{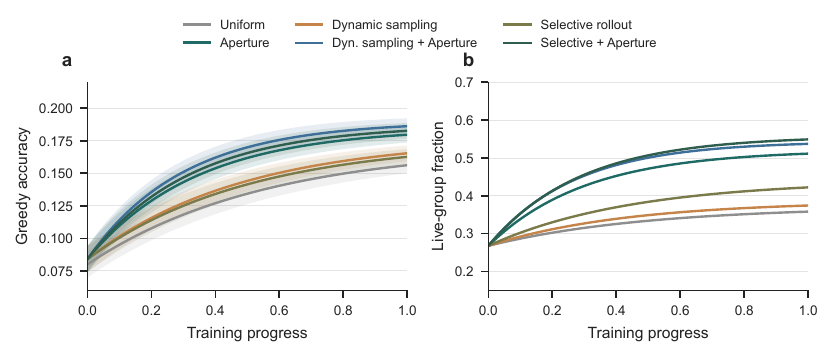}
  \caption{Qwen3-1.7B on logic.
  Axes match Figure~\ref{fig:results}.
  The live-group gap is similar to math.
  The accuracy level is lower because the held-out puzzles, and ARC-AGI-2 in particular, are harder for this model.}
  \label{fig:logic}
\end{figure}

\subsection{Seeds}
\label{app:seeds}

Table~\ref{tab:seeds} lists the three seeds behind the 1.7B macros.
The seed range on math is a few thousandths for every method, smaller than the gap between uniform GRPO and \method{}.
Logic is slightly noisier on the hard band and still separates the prior from uniform sampling in every seed.
We do not report a paired significance test.
With three seeds a test would be a statement about these three runs, and the tables already show the runs.

\begin{table}[ht]
  \caption{Per-seed macro-accuracy, Qwen3-1.7B.
  The mean of each row is the corresponding entry of Tables~\ref{tab:math} and~\ref{tab:logic}.}
  \label{tab:seeds}
  \centering
  \small
  \begin{tabular}{lcccccc}
\toprule
& \multicolumn{3}{c}{Math} & \multicolumn{3}{c}{Logic} \\
\cmidrule(lr){2-4}\cmidrule(lr){5-7}
Method & seed 0 & seed 1 & seed 2 & seed 0 & seed 1 & seed 2 \\
\midrule
Uniform & 0.383 & 0.380 & 0.369 & 0.172 & 0.168 & 0.173 \\
Aperture & 0.392 & 0.397 & 0.396 & 0.182 & 0.185 & 0.188 \\
Dyn.\ sampling & 0.385 & 0.384 & 0.384 & 0.175 & 0.181 & 0.175 \\
Dyn.\ + Aperture & 0.399 & 0.399 & 0.406 & 0.193 & 0.188 & 0.189 \\
Selective & 0.383 & 0.380 & 0.384 & 0.175 & 0.174 & 0.176 \\
Sel.\ + Aperture & 0.399 & 0.398 & 0.398 & 0.183 & 0.195 & 0.184 \\
\bottomrule
\end{tabular}

\end{table}

\subsection{Another model family}
\label{app:llama}

Table~\ref{tab:llama} repeats the math protocol with Llama-3.2-3B.
The absolute accuracy is lower than Qwen3-1.7B.
The ordering is the same: the prior alone beats uniform sampling and dynamic sampling, and the product of the prior with dynamic sampling is the best row.
We read this as evidence that the unanimity identity, rather than a quirk of the Qwen tokenizer or pretraining mix, is what the prior is using.
It is one model and one domain.
Appendix~\ref{app:discussion} says what a failure of transfer looks like on the domain axis, where we did measure it.

\begin{table}[ht]
  \caption{Llama-3.2-3B, math, same metrics as the Qwen runs.}
  \label{tab:llama}
  \centering
  \begin{tabular}{lccc}
\toprule
Method & Acc & AUC & Live \\
\midrule
Uniform & 0.316 & 0.288 & 42.1\% \\
Aperture & 0.331 & 0.308 & 56.6\% \\
Dyn.\ sampling & 0.322 & 0.296 & 43.4\% \\
Dyn.\ + Aperture & 0.336 & 0.314 & 58.1\% \\
Selective & 0.320 & 0.293 & 48.8\% \\
Sel.\ + Aperture & 0.333 & 0.311 & 59.4\% \\
\bottomrule
\end{tabular}

\end{table}

\subsection{Other online schedulers}
\label{app:compose}

Table~\ref{tab:compose} multiplies $w$ onto a moving-edge filter \citep{wu2026edge} on Qwen3-1.7B math.
The filter improves on uniform sampling by itself, and it improves further when the prior is attached.
The live rate moves by more than the accuracy, which is the same signature as in Table~\ref{tab:math}.
We did not sweep the filter's own threshold.
The product is positive on both accuracy and the live rate.

\begin{table}[ht]
  \caption{The prior as a multiplier on a moving-edge filter.
  Qwen3-1.7B, math.}
  \label{tab:compose}
  \centering
  \begin{tabular}{lcc}
\toprule
Method & Acc & Live \\
\midrule
Uniform & 0.377 & 45.8\% \\
Moving-edge filter & 0.385 & 51.2\% \\
Moving-edge + Aperture & 0.397 & 62.4\% \\
\bottomrule
\end{tabular}

\end{table}

\subsection{Sweeps}
\label{app:sweeps}

Probe size (Table~\ref{tab:psize}) has a knee at $|P|=\nprobe$.
A probe of $256$ prompts is only a small step above uniform sampling.
A probe of $2048$ prompts adds little over $512$ and costs four times the probe rollouts.
Probe length (Table~\ref{tab:pepoch}) has the same shape: two epochs are under-trained relative to the main recipe, and ten epochs are close to six.

The group size inside the score (Table~\ref{tab:gmiss}) is a direct check of Proposition~\ref{prop:reverse}.
Training always uses $G=8$.
Scoring with $G'=4$ or $G'=16$ is worse than scoring with $G'=8$, by about half a point of accuracy.
The damage is limited because many prompts are extreme enough that every reasonable $G$ ranks them together.
The prompts near the reversal, which the proposition isolates, are a slice of the pool rather than the whole pool.

The mixture (Table~\ref{tab:lambda}) is best at $\lambda=\mixLam$.
$\lambda=0$ is uniform sampling.
$\lambda=1$ gives the interior prompts no floor and loses accuracy even though its live rate in the main ablation is high: the run concentrates on prompts that are opening and under-samples prompts that are already useful.
The power (Table~\ref{tab:alpha}) is best at $\alpha=\mixAlpha$.
$\alpha=1$ is peaked, and $\alpha=0.25$ is close to uniform.

\begin{table}[ht]
  \caption{Probe size on Qwen3-1.7B.
  Cells are math accuracy and logic accuracy.}
  \label{tab:psize}
  \centering
  \begin{tabular}{lc}
\toprule
Probe size & Math / Logic acc \\
\midrule
$|P|=256$ & 0.384 / 0.175 \\
$|P|=512$ & 0.395 / 0.185 \\
$|P|=1024$ & 0.397 / 0.186 \\
$|P|=2048$ & 0.398 / 0.187 \\
\bottomrule
\end{tabular}

\end{table}

\begin{table}[ht]
  \caption{Probe length, Qwen3-1.7B math.}
  \label{tab:pepoch}
  \centering
  \begin{tabular}{lc}
\toprule
Probe length & Math acc \\
\midrule
2 epochs & 0.386 \\
4 epochs & 0.391 \\
6 epochs & 0.395 \\
10 epochs & 0.396 \\
\bottomrule
\end{tabular}

\end{table}

\begin{table}[ht]
  \caption{Group size $G'$ used inside \eqref{eq:score}, with training still at $G=8$.
  Qwen3-1.7B math.}
  \label{tab:gmiss}
  \centering
  \begin{tabular}{lc}
\toprule
Score group size & Math acc \\
\midrule
$G'=4$ & 0.390 \\
$G'=8$ & 0.395 \\
$G'=16$ & 0.391 \\
\bottomrule
\end{tabular}

\end{table}

\begin{table}[ht]
  \caption{Mixture weight $\lambda$, Qwen3-1.7B math.}
  \label{tab:lambda}
  \centering
  \begin{tabular}{lc}
\toprule
Mixture & Math acc \\
\midrule
$\lambda=0$ (uniform) & 0.377 \\
$\lambda=0.5$ & 0.386 \\
$\lambda=0.75$ & 0.395 \\
$\lambda=1.0$ & 0.389 \\
\bottomrule
\end{tabular}

\end{table}

\begin{table}[ht]
  \caption{Score power $\alpha$, Qwen3-1.7B math.}
  \label{tab:alpha}
  \centering
  \begin{tabular}{lc}
\toprule
Power & Math acc \\
\midrule
$\alpha=0.25$ & 0.387 \\
$\alpha=0.5$ & 0.395 \\
$\alpha=1.0$ & 0.390 \\
\bottomrule
\end{tabular}

\end{table}

\subsection{Reusing the prior}
\label{app:transfer}

Table~\ref{tab:transfer} asks how much of the prior is about the pool and how much is about the model.
A prior fit on Qwen3-1.7B and used to train Qwen3-4B on math reaches \xferCross{}, between the uniform 4B run at \MAccG{} and the in-regime 4B prior at \MAccA{}.
Most of the in-regime gain survives the change of scale inside one family.
A prior fit on math and used to train the 1.7B model on logic reaches \xferDom{}, a smaller step from the uniform logic run at \lAccG{} toward the in-regime logic prior at \lAccA{}.
The unanimity formula is the same in both domains.
The prompts that open are not.
We refit the probe when the domain changes, and we would refit it for a new reward or a new group size.

\begin{table}[ht]
  \caption{Reused priors.
  ``In-regime'' refits the probe on the model and pool being trained.}
  \label{tab:transfer}
  \centering
  \begin{tabular}{lc}
\toprule
Prior & Acc \\
\midrule
In-regime 4B math & 0.497 \\
1.7B prior, 4B math & 0.490 \\
Uniform 4B math & 0.481 \\
In-regime 1.7B logic & 0.185 \\
Math prior, 1.7B logic & 0.176 \\
Uniform 1.7B logic & 0.171 \\
\bottomrule
\end{tabular}

\end{table}

\subsection{How long the static score stays aligned}
\label{app:temporal}

The prior is computed once.
Table~\ref{tab:temporal} recomputes the score at later checkpoints, using the same formula with the current policy in place of $\thprobe$, and reports the Spearman correlation with the initial score on the math pool, Qwen3-1.7B.
The correlation is \spearEarly{} at $25\%$ of training and \spearEnd{} at the end.
The ranking decays, and it does not collapse.
High-score prompts remain high-score prompts more often than a fresh shuffle would allow.
This is why a static multiplier is still visible late in Figure~\ref{fig:results}, and why we do not claim the prior should be frozen for a run many times longer than the eight epochs we used.
A natural extension, which we do not run here, is to refresh the probe when the correlation falls below a threshold.

\begin{table}[ht]
  \caption{Spearman correlation between the initial \method{} score and a score recomputed at the current policy.
  Qwen3-1.7B, math.}
  \label{tab:temporal}
  \centering
  \begin{tabular}{lc}
\toprule
Training progress & Spearman $\rho$ \\
\midrule
25\% & 0.840 \\
50\% & 0.770 \\
75\% & 0.690 \\
100\% & 0.610 \\
\bottomrule
\end{tabular}

\end{table}

\subsection{Where the prior does not help}
\label{app:failure}

We built an easy slice by taking the math prompts whose base pass rate, estimated with $32$ rollouts, is at least $0.8$, and evaluating greedy accuracy on a matched held-out set of the same filter.
Table~\ref{tab:failure} shows the result.
Uniform GRPO scores \failAccG{}.
\method{} scores \failAccA{}.
Dynamic sampling is slightly higher than uniform, and the product with the prior gives the gain back.
On this slice the failure mode the score is designed for is scarce: groups are unanimous because the prompts are solved, and downweighting them removes practice on prompts the evaluation still contains.
The main pools are not this slice.
The result is the boundary of the claim, and it is the reason the mixture keeps a uniform component rather than discarding every low-score prompt.

\begin{table}[ht]
  \caption{Easy slice, Qwen3-1.7B.
  Prompts are already solved at high pass rate under the base model.}
  \label{tab:failure}
  \centering
  \begin{tabular}{lc}
\toprule
Method & Easy-slice acc \\
\midrule
Uniform & 0.742 \\
Aperture & 0.739 \\
Dyn.\ sampling & 0.744 \\
Dyn.\ + Aperture & 0.741 \\
\bottomrule
\end{tabular}

\end{table}

\section{Discussion}
\label{app:discussion}

\paragraph{What the score refuses.}
Proposition~\ref{prop:mono} sets $s=0$ when the probe pass rate falls.
A prompt that drops from $0.9$ to $0.5$ is becoming a split group, which is exactly the event the live-rate curves reward, and the score ignores it.
We made that choice because a falling pass rate on the probe is ambiguous.
It can mean the prompt is being forgotten, in which case more weight would accelerate the drop, or it can mean the probe was noisy.
The positive-part takes the conservative side.
A signed score that upweights movement toward $1/2$ from either direction is a different method, and it would need a probe stable enough to tell forgetting from noise.
We do not have that separation.

\paragraph{Binary rewards.}
Lemma~\ref{lem:dead} uses a Bernoulli reward.
A rubric with partial credit, a unit-test suite with several weighted tests, or a verifier that returns a real number in $[0,1]$ can have within-group variance zero for reasons other than unanimity of a single bit, and it can have small variance without being unanimous.
The expected-energy identity in Lemma~\ref{lem:energy} has an analogue for any bounded i.i.d.\ reward, with $p(1-p)$ replaced by the reward variance, but the closed form $p^G+(1-p)^G$ does not survive.
Extending \method{} to real-valued rewards means estimating that variance at the two endpoints, or estimating the probability that the empirical variance is zero, which for continuous rewards is an event of a different kind.
We stayed inside exact-match tasks so the formula we plot is the formula we optimize.

\paragraph{The probe can overfit its subset.}
The probe trains on $\nprobe$ prompts for $\epProbe$ epochs.
Those prompts are also in $D$, so they are seen more than the prompts outside $P$.
Their endpoint $\tilde\rho_1$ can be high because they were memorized rather than because the recipe is about to solve them.
Two checks limit the damage.
First, the score of a memorized prompt that races to the ceiling is killed by $\nu_G(\tilde\rho_1)$, which is the same mechanism that kills prompt C in the witness.
Second, the hold-out benchmarks are not in $P$.
A probe that only memorizes $P$ has to help MATH-500, AIME, and the held-out puzzles to move the numbers we report.
It does.
A probe the size of the full pool would remove this concern and would also remove the reason to probe rather than to train.
The size sweep is the practical compromise.

\paragraph{Staleness and refresh.}
Table~\ref{tab:temporal} is the quantitative version of a complaint any static prior deserves.
Eight epochs is long enough for the ranking to loosen and short enough that it still agrees with itself.
Runs that continue until the live rate collapses for a second time, because the remaining prompts have all sealed at $1$, are outside our evidence.
The easy-slice result is a preview of that regime.
Refreshing the two endpoints at the midpoint, and rebuilding $w$, is the obvious patch.
It doubles the evaluation cost and keeps the probe idea.
We leave it as a follow-up because the static version already has a clean comparison with online methods that refresh every step.

\paragraph{Relation to advantage shaping.}
Methods that rewrite the advantage on a unanimous group are solving the same waste with the opposite commitment \citep{le2025noprompt,liu2025drgrpo,wang2025entropy}.
They keep the prompt in the batch and manufacture a direction to update.
\method{} keeps the advantage honest and changes the batch.
The two are compatible: a shaped advantage can be applied to the groups that the prior still draws.
We do not combine them, because the experiment would no longer say whether the gain came from the prompt distribution.
The live rate would also become a worse diagnostic, since a shaped unanimous group is no longer a zero update.

\paragraph{Relation to zero-rollout difficulty.}
A difficulty prior that looks only at the base model, or only at the text of the prompt, is cheaper than a probe \citep{sha2026thinkprior,shi2025adarf,ye2025limo}.
It is also, by Lemma~\ref{lem:dead}, a monotone function of something other than the future drop in $\cfail_G$.
If that something is the current pass rate, Section~\ref{sec:matched} is the reason it is not enough: inside a pass-rate tercile the future still moves accuracy by several points.
If it is a feature of the prompt text, it can correlate with opening, and a probe is then an expensive way to read a cheap feature.
We do not have a negative result of that form.
We have a positive result for the probe against pass-rate controls, which are the strongest snapshot of the same verifier the training run will use.

\paragraph{Domains and scale.}
The logic pool is smaller, the gain is similar in absolute accuracy, and the cross-domain reuse is the weak cell of Table~\ref{tab:transfer}.
We would not ship a math prior onto a code verifier, or a $G=8$ prior onto a $G=64$ training run, without refitting.
Proposition~\ref{prop:reverse} is the $G$ half of that warning.
The domain half is empirical.
Scale inside Qwen3 is kinder than domain transfer: most of the 4B gain is available from the 1.7B probe.
That is a statement about these two checkpoints, both trained with the same recipe on the same pool.
It is not a statement about a 70B model we did not run.

\paragraph{What would falsify the claim.}
The claim is that a probe forecast of $\nu_G(\tilde\rho_1)[\cfail_G(\tilde\rho_0)-\cfail_G(\tilde\rho_1)]_+$ ranks prompts by the accuracy and live rate they produce under this GRPO recipe, beyond the ranking induced by $\tilde\rho_0$.
Table~\ref{tab:within} is the falsification test we actually ran.
It would have failed if the high-score half had matched or lost to the low-score half inside a tercile.
A second falsification would be a pool on which headroom-gain beats the full score on both accuracy and live rate.
Table~\ref{tab:ablation} does not show that.
A third would be a group size at which the identity in Lemma~\ref{lem:dead} is the wrong waste: continuous rewards, or a learned verifier whose ties are rare.
That case is outside the experiments, and the discussion above says why the formula would have to change.

\end{document}